\documentclass[pdflatex,sn-mathphys]{sn-jnl}
\usepackage{enumerate}
\usepackage[shortlabels]{enumitem}
\usepackage{amsmath}
\usepackage{algpseudocode}

\DeclareMathOperator*{\argmax}{arg\,max}
\DeclareMathOperator*{\argmin}{arg\,min}
\DeclareMathOperator{\conv}{Conv}
\DeclareMathOperator{\ent}{Ent}
\DeclareMathOperator{\mmd}{mmd}
\DeclareMathOperator{\spt}{spt}
\DeclareMathOperator{\Proj}{Proj}
\jyear{2021}%

\theoremstyle{thmstyleone}%
\newtheorem{theorem}{Theorem}
\newtheorem{proposition}[theorem]{Proposition}%

\newtheorem{corollary}{Corollary}
\newtheorem{problem}{Problem}
\theoremstyle{thmstyletwo}%
\newtheorem{remark}{Remark}%
\newtheorem{idea}{Idea}
\newtheorem{question}{Open Question}

\theoremstyle{thmstylethree}%
\newtheorem{definition}{Definition}%
\newtheorem{assumption}{Assumption}

\begin{document}

\title[Measure pre-conditioning]{On the impact of measure pre-conditionings on general parametric ML models and transfer learning via domain adaptation}

\author*[1,2]{\fnm{Sanchez Garcia} \sur{Joaquin}}\email{joaqsan@math.utoronto.com}


\abstract{We study a new technique for understanding convergence of learning agents under small modifications of data. We show that such convergence can be understood via an analogue of Fatou's lemma which yields $\Gamma$-convergence. We show it's relevance and applications in general machine learning tasks and domain adaptation transfer learning.}


\maketitle
\tableofcontents
\markboth{}{}
\section{Introduction}\label{sec1ML}
Recent progress in the use of optimal transportation techniques for machine learning in domain adaptation \cite{CourtyFlamary} and development of Wasserstein Generative adversarial networks \cite{WGAN} have helped our understanding of potential learning derived from theoretic properties of the underlying data. The topic of optimal transportation has grown significantly in recent years (see \cite{Villani}, \cite{VillaniOldAndNew}, \cite{Santambrogio} and references therein). \\
Machine learning models aim to solve a task (to prescribed accuracy) using only the information of known data (training set). In this context it is preferred to have non-parametric models over parametric statistical families. \\
In this document we explore an idea that we call \textbf{measure pre-conditioning} the training data which consists in modifying the statistical model in order to improve performance of algorithms while preserving the limiting model. One can argue that measure pre-conditioning implicitly imposes unjustified structure to a problem but the idea is that measure pre-conditioning will simplify computations and ensure convergence to the original model. For example measure pre-conditioning one of the measures may allow using optimal transportation techniques to adapt a domain which would otherwise be very costly, this would yield a desired training in a task with little information. \\ We use the terminology ``measure pre-conditioning'' as the technique reminds us of pre-conditioning matrices from linear algebra and optimization.
\subsection{Organization of this document}
\subsection{Relation to literature}
The authots of \cite{CourtyFlamary} develop the idea of optimal transport domain adaptation on which a linear approximation of the transport map is used to infer labels on target domain and \cite{Courty} developed CO-OT, a technique on which optimal transport is not only done between source and target domains of data but in the space of data and labels. \\
Recently \cite{Meta} developed the META-optimal transport technique which by pre-solving an optimization problem improves on the optimal transport efficiency. In this work, the idea is similar: can we modify training sets to ensure properties of learning? The modifications considered in this document, differ from the ones on \cite{Meta} as we only consider measure pre-conditioning data without establishing a minimization purpose beforehand. These techniques should remind the reader of the concept of preconditioning in optimization, on which one modifies a matrix via a correct scaling to benefit the algorithm computations. In the same fashion, here one modifies the measure associated to a training set to benefit statistical properties of the learning agent. 
\subsection{Necessity of non-parametric measure pre-conditioning techniques}
The need for non-parametric measure pre-conditioning techniques arises from the modeller's attempt to not intervene in the learning while improving it's computational performance. Measure pre-conditioning is posed in this document as a general technique and it is the modeller's task to determine which pre-conditioning is useful for their own goal. In section \ref{definitions} we give several examples with different goals in mind.

\section{Measure pre-conditionings}
In this section we introduce the main concept and discuss several possible ``measure pre-conditionings''. In this context a measure pre-conditioning will be a technique to manipulate data in order to obtain a ``nicer'' measure. For example, we can regularize our problem to obtain a measure that is absolutely continuous with respect to Lebesgue or a measure that has a different type of support. \\
Measure pre-conditioning is also similar to parameter fitting for curves. In the case of real variable one attempts to infer information from isolated data points by first creating a continuous (typically smooth) curve joining the points. Pre-conditioning between points in $\mathbb{R}$ has drawbacks (overfitting, high-variation, etc) and so will measure pre-conditioning (see section \ref{Problems}). Measure pre-conditionining will have the advantage of enabling stronger techniques to infer learning as we will see throughout the paper. We start by defining several possible measure pre-conditioning techniques and analyzing their properties. 
\begin{problem} (General measure pre-conditioning problem for independent identically distributed data) \\
Let $(X_1,X_2, \dots, X_n)$ be a sample, that is $\{X_i\}_{i =1}^n$ is a set of independent identically distributed data such that $X_1 \sim \mu$. Suppose that the sample will be used to train a machine learning model, the measure pre-conditioning problem is to find a good way to obtain a measure $\Tilde{\mu}_n$ from the sample such that $\Tilde{\mu}_n$  improves performance of the model or the computational cost of the algorithms while keeping the most relevant features of the problem intact. 
\end{problem}
As such, this measure pre-conditioning problem is not mathematically well posed, as we haven't defined what ``improves performance of the model''  or `` keeping the most relevant features'' mean. Performance improvement can be done in several ways: simplification of algorithms, computational cost, control on domain adaptation or even yielding mathematical properties for the learning agent. All of these type of improvements are valid and impactful in machine learning research. The aim of this paper is to analyze how different measure pre-conditionings impact model performance. \\
\section{A mathematical framework admitting pre-conditioning} \label{mathframework}
Let us start with a basic framework from Machine Learning models in order to be able to define measure-preconditioning and show it's relevance. The simplest case is the minimization over all fitting functions $f$ within a class of fitters $\mathcal{C}$ minimizing the expected value of the loss function $L$ measuring the loss of fitting the random variable $Y$ with the variable $X$ via $f(X)$.
\subsection{Formulation of the problem} \label{formulation}
\begin{problem} \label{context}
Let $\Omega \subseteq \mathbb{R}^n$ be convex and compact. Assume we have data $X \sim \mu \in \mathcal{P}(\Omega)$ and we aim to do a Machine-learning model towards a dependent variable $Y \in \mathcal{Y}$ where $(\mathcal{Y}, d_Y)$  is a separable complete metric space, we denote by $\pi \in \mathcal{P}(\Omega \times \mathcal{Y})$ the joint distribution of $(X,Y).$ Given $L: \mathcal{Y} \times \mathcal{Y} \to \mathbb{R}$ (called a loss function), let $\mathcal{C} \subseteq \mathcal{Y}^{\Omega}$ and assume $d$ is a distance function on $\mathcal{C}$, the $\mathcal{C}-$ optimal model for $L$ under $\pi$ is the following non-linear program 
\begin{equation} \label{piProblem}
    \argmin_{f \in \mathcal{C}} \mathbf{E}_{\pi} \left[ L(f(x),y) \right]. 
\end{equation} 
Now assume we don't know the full model $\pi$ but we have a training sample, i.e. we have $(X_1,Y_1), \dots, (X_n,Y_n) \sim \pi$, statistically we know the values on the sample but not the full distribution. Assume we approximate $\pi$ using the sample via a probability measure $\pi_n$, the associated $\mathcal{C}-$model for $L$ under $\pi_n$ reads 
\begin{equation} \label{pinProblem}
        \argmin_{f \in \mathcal{C}} \mathbf{E}_{\pi_n} \left[ L(f(x),y) \right].
\end{equation}
\end{problem}
This formulation immediately give rise to the following questions
\begin{enumerate}[i]
    \item \label{question1} If $L$ and $\mathcal{C}$ are fixed, what conditions on $\pi_n$ ensure that the minimizer in \eqref{pinProblem} approaches \eqref{piProblem}? In what topology? 
    \item \label{question2} What properties could \eqref{pinProblem} have that \eqref{piProblem} may lack? 
    \item \label{question3} Given a choosing of $\pi_n$'s, could we find sequences $L_n$'s and $\mathcal{C}_n$ so that the computations on the $\mathcal{C}_n-$ problem with loss function $L_n$ associated to $\pi_n$ converge to \eqref{piProblem}? Could these problems improve the algorithmic performance? 
\end{enumerate}

\begin{idea} (\textbf{Measure pre-conditioning}) \\
    A measure pre-condition is a way to define $\pi_n$ from the sample $(X_1,Y_1), \dots, (X_n,Y_n)$ such that the associated $\mathcal{C}$-problem with loss function $L$ has improved performance in any way while preserving the convergence of minimizers of \eqref{pinProblem} to that of \eqref{piProblem}.
\end{idea}
\subsection{Convergence of the learning problem}
Our main focus will be answering: when do minimizers of \eqref{pinProblem} converge to minimizers of \ref{piProblem} and in which way?. \\

We first notice that in many situations it is possible to obtain the same total loss under convergence of the measures (without necessarilly having convergence of minimizers), this situation is rather general and known and is not the main question in the ML community but it gives a good starting point for the techniques used in this document. For many applications it is enough to know convergence of the total loss and so we exemplify conditions that yield such convergence. 
\begin{proposition} \label{list} (Standard convergence results on total loss (not minimizers)) \\
\begin{enumerate}    
\item If $  \lvert \lvert L \rvert \rvert_{\infty} < \infty $ or if $\spt(\mu)$ is compact, 
    \[
\lvert \mathbf{E}_{\pi_n} [L(f(\overline{X}),Y)] - \mathbf{E}_{\pi} [L(f(\overline{X}),Y)] \rvert \leq \lvert \lvert \pi_n - \pi \rvert \rvert_{TV}.
\]
 \item Given $f \in \mathcal{C}$ if $(x,y) \to L(f(x),y) $ is Lipschitz, then 
    \[ 
\lvert \mathbf{E}_{\pi_n} [L(f(\overline{X}),Y)] - \mathbf{E}_{\pi} [L(f(\overline{X}),Y)] \rvert \leq d_1(\pi_n,\pi).
\]
 
 \item If $L$ is $\mathcal{C}^ 2$  and $ \lvert \lvert \frac{\partial L}{\partial 1} \rvert \rvert < \infty $ then 
    \[ 
\lvert \mathbf{E}_{\pi_n} [L(f(\overline{X}),Y)] - \mathbf{E}_{\pi} [L(f(\overline{X}),Y)] \rvert \lesssim \lvert \lvert \pi_n - \pi \rvert \rvert_{TV} + \sup_{x \in \Omega} d(f^*(x), f_n^{*}(x))
\] 
\item 
    If $\mathcal{C}$ is a compact class on $C(\mathcal{Y})$, and $\pi_n \to \pi$ in $d_1$ then along a subsequence $n_k$
    \[
    \mathbf{E}_{\pi_{n_k}}[L(f_{n_k}^*(X),Y)] \to \mathbf{E}_{\pi}[L(f^*(X),Y)] 
    \]
    where $f_{n_k}^*$ is the $\mathcal{C}-$optimizing argument for $\pi_n$ and $f^*$ is the $\mathcal{C}-$optimizing argument for $\pi$.
\end{enumerate}
\end{proposition}
The proof of proposition \ref{list} is direct and hence omitted.
\subsection{The main question}
Measure preconditioning modifies the minimization problem at level $n$, i.e. it changes the structure of the measure used to evaluate loss with a sample of size $n$. If the model was unchanged we would expect convergence of the learning agent trained with the sample of size $n$, i.e. $f_n^*$ to the best fit with respect to the loss for the parametric distribution $f^*$. If measure pre-conditioning modifies the measure at level $n$, the true question is when and in which ways does $f_n^* \to f^*$ ?. \\
To answer the convergence of minimizers, as it is usual in functional analysis and economics, we introduce $\Gamma$-convergence.
\subsubsection{Main Theorem}
We present an informal version of the main theorem of the work. This informal version corresponds to the rigorous statements answered in Theorem \ref{flrsmin}, Proposition \ref{cases} and section \ref{DomainAdaptationSection} \begin{theorem} \label{mainTheoremvague}
    Full learner recovery system is a concept that allows us to show convergence of learning agents to the ideal parametric agent in cases not covered previously in the literature. This concept allows us to generalize stability arguments for less regular losses and a bigger class of classification/regression problems. Full learner recovery systems are general enough to be applied to several settings in Machine-Learning, including Domain Adaptation transfer learning. These systems explain many phenomena in ML-research where convergence is improved. Full learner recovery systems give a guideline on how and when to modify training data without disturbing the original problem.
\end{theorem}
The formulation of Theorem \ref{mainTheoremvague} is not mathematically precise, we dedicate this work to make the Theorem rigorous and prove it in the subsequent sections. \\
We start with the introduction of the main mathematical tool.
\subsection{A version of the envelope Theorem}
    \begin{definition} \label{gammaconv} ($\Gamma$-convergence on a metric space) \\
    Let $(X,d)$ be a metric space and let $F_j,F: X \to \mathbb{R} \cup \{\pm \infty\}$,we say $F_n$ $\Gamma$-converges to $F$, denoted $F_n \xrightarrow{\Gamma} F$ if and only if the following two conditions hold \begin{enumerate}[I]
        \item \label{Gammaconv1} For all sequences $ \{x_j\}$ such that $ x_j \xrightarrow{d} x$ we have \[
        \liminf_{j \to \infty} F_j (x_j) \geq F(x).
        \]
        \item \label{Gammaconv2} For every $x \in X$ there exists a sequence $ x_j \xrightarrow{d} x$ such that 
        \[
        F(x) \geq \limsup_{j \to \infty} F_j(x_j).
        \]
    \end{enumerate} 
    \end{definition}
    \begin{remark} \label{TopologyNotMetric} The most general definition for $\Gamma$-convergence is one where $X$ is assumed to be a topological space and not necessarily metric. The definition presented above (Definition \ref{gammaconv}) is the sequential-definition. We have chosen the sequential definition as it simplifies the theory significantly, knowing that some important examples that we have in mind are only topological spaces on which the $\Gamma$-limit is defined via \begin{equation} \label{GammaTop}
       \Gamma-\lim_{n \to \infty} F_n (x)  = \sup_{U \in N(x)} \liminf_{n \to \infty} \inf_{y \in U} f_n(y).
    \end{equation}
    In some of the examples below the underlying convergence will not correspond to a metric space, on which one must think of \eqref{GammaTop} instead of \eqref{Gammaconv1} and \eqref{Gammaconv2}.
    \end{remark}
The motivation behind the definition of $\Gamma$-convergence is that minimizers converge to minimizers, the content of the following theorem from \cite{GammaConv}:
 \begin{theorem}  ($\Gamma$-convergence and minima) \\
 Let $(X,d)$ and $F_j,F$ be as in Definition \ref{gammaconv}, then 
\begin{enumerate}
    \item If \ref{Gammaconv1} from definition \ref{gammaconv} is satisfied for all $x \in X$ and $K$ is a compact subset of $X$ then \begin{equation}
        \inf_{K} F \leq \liminf_{j \to \infty} \inf_K F_j
    \end{equation}
    \item Similarly, if \ref{Gammaconv2} from definition \ref{gammaconv} is satisfied and $U$ is an open subset of $X$ then 
    \begin{equation}
        \limsup_{j \to \infty} \inf_U F_j \leq  \inf_U F
    \end{equation}
\end{enumerate}
\end{theorem}
This Theorem can be found as \cite[Proposition 1.18]{GammaConv}. Finally we recall one more Theorem from \cite{GammaConv}. We say that a sequence $\{F_j\}$ of functions on a metric space $(X,d)$ is equi-mildly coercive if there exists a non-empty compact set $K$ such that \begin{equation*}
    \inf_{X} F_j = \inf_{K} F_k \text{  for all } j. 
\end{equation*}
\begin{theorem} \label{Gammaminimizers} (Minimizers and $\Gamma$-limits) \\
    In a metric space $(X,d)$ if $\{F_j\}$ is equi-mildly coercive and $F_n \xrightarrow{\Gamma} F$ then \begin{equation} \label{MinimizersExistGamma}
        \min_{X} F = \lim_{ j \to \infty} \inf_{K} F_j
    \end{equation}
    Furthermore, every limit point of a sequence of minimizers of \eqref{MinimizersExistGamma} is a minimizer of $F$.
\end{theorem}
For a proof see \cite[Theorem 1.21]{GammaConv}. With the theory in hand we take a general approach to answer the questions \eqref{question1} and \eqref{question2}. Instead of a constructive proof to find the optimal topologies (on $\mathcal{C}$ and $\mathcal{P}(X \times Y)$) we reformulate the convergence problem for it to satisfy the hypothesis of Theorem \ref{gammaconv}. This way we can relate to classical problems by looking at the given topologies of each framework and verifying the hypothesis.  \\
Going back to the framework of Problems \eqref{piProblem} and \ref{pinProblem}, we want to be able to recover minimizers from our measure conditioning. We note the interaction of the class of fitters $\mathcal{C}$, the loss function $L$ and the mode of convergence of the conditioners that we choose to evaluate, henceforth it is logical to check conditions for them as a collective, rather than separately. This motivates the following definition. 
\begin{definition}\label{flrs} (Full learner recovery system) \\
In the context of Problem \ref{context}, we say that $(\mathcal{C}, d, L, \xrightarrow{m})$ forms a full learner recovery system if it holds that \begin{enumerate}
    \item If $\pi_n \xrightarrow{m} \pi$ for all $d$-converging sequences sequences $f_n \xrightarrow{d} f$,  we have \begin{equation}
        \liminf_{n \to \infty} \mathbf{E}_{\pi_n} [L(f_n(X),Y)] \geq \mathbf{E}_{\pi} [L(f(X),Y)].
    \end{equation}
    \item If $\pi_j \xrightarrow{m} \pi$ and for every $f \in \mathcal{C}$ there exists a sequence $ f_j \in \mathcal{C}$, such that $f_j \xrightarrow{d} f$ and \begin{equation}
        \mathbf{E}_{\pi} [L(f(X),Y)] \geq \limsup_{j \to \infty} \mathbf{E}_{\pi_j}[L(f_j(X),Y)]
    \end{equation}
\end{enumerate}
\end{definition}
\begin{remark} In analytical terms, these conditions ensure 2-sided Fatou-Lemmas for integration with respect to $L$ on the first coordinate. 
\end{remark}
$\Gamma$-convergence can be also used to address the existence of minimizers of the parametric model but that is not the approach of this work, we assume existence of minimizers of the limiting problem and study recovery sequences, from now on we assume the existence of a unique minimizers for \eqref{pinProblem}.
\begin{theorem} \label{flrsmin} If $(\mathcal{C}, d, L, \xrightarrow{m})$ forms a full learner recovery system (Definition \ref{flrs}) where $(C,d)$ is a compact metric space, assume the limiting problem from \ref{piProblem} has a solution $ f \in \mathcal{C}$, then there exists a sub-sequence $\{f_{n_k}\}$ of $\{f_n\} \in \mathcal{C}$ such that \begin{equation*}
    f_{n_k} \in \argmin_{f \in C} \mathbf{E}_{\pi_{n_k}}[L(f(X),Y)]
\end{equation*}
such that as $k \to \infty$, $f_{n_k} \xrightarrow{d} f$ and 
\begin{equation} \label{minsconvergetoo}
    \mathbf{E}_{\pi_{n_k}} [L(f_{n_k}(X),Y)] \to \mathbf{E}_{\pi} [L(f(X),Y)) ].
\end{equation}
\end{theorem}
\begin{proof}
    The definition of $(\mathcal{C}, d, L, \xrightarrow{m})$ forming a full learner recovery system is such that  $\mathbf{E}_{\pi_n}^L \xrightarrow{\Gamma} \mathbf{E}_{\pi}^L$, i.e. by taking the functional $   F_n (f): \mathcal{C} \to \mathbb{R}$, defined via 
    \begin{equation*}
        F_n (f) := \mathbf{E}_{\pi_n}[L(f(X),Y)],
    \end{equation*}
the definition \ref{flrs} is equivalent to $F_n \xrightarrow{\Gamma} F$. By compactness of $\mathcal{C}$ we get the hypothesis for Theorem \ref{Gammaminimizers} so we get the thesis.
\end{proof}
In many cases $\mathcal{C}$ is not necessarily compact. The assumption of compactness simplifies the arguments but the argument above can be obtained without compactness of $\mathcal{C}$ if instead one assumes equi-mild-coercivity of $\{ F_n\}$, that is there exists a compact set $K$ for which all $F_n$'s satisfy $\inf_C F_n = \inf_K F_n$. See  \cite[Theorem 1.21]{GammaConv}, we instead assume compactness of $\mathcal{C}$ to avoid this subtlety.
\begin{remark} Evidently the statement of Theorem \ref{flrsmin} is useless unless we explore examples and explain the ideas and how to use it. So far, we have just re-written the problem so that we can conclude (subsequential) convergence of learned agents by checking a modified version of Fatou's Lemma. This rewriting allows us to cover different cases at the same time, as we do in the following examples.
\end{remark}
The goal of this list is not to be exhaustive but to show the many different formulations that can be included in Definition \ref{flrs}. Notice that checking Definition \ref{flrs} involves only studying a two sided version of Fatou's Lemma that can be corroborated in every particular case. Once one establishes that the given ML problem of the form \eqref{pinProblem} and \eqref{piProblem} are indeed a full recovery system with $\{ \pi_n \},\pi, \mathcal{C},L$ one has ensured convergence of minimizers (which amounts to perfect approximation of the model). \\
In the following proposition we show the wide range of options one has for full recovery systems, although the $d$-convergence in some items of the following proposition are not necessarily with respect to a metric, we have in mind Remark \ref{TopologyNotMetric}.
\begin{proposition} \label{cases} The following are full learner recovery systems
\begin{enumerate} \item Let $K \subset \mathbb{R}^p$ be compact, $\mathcal{C}$ a compact subset of $ \{ f:\mathbb{R}^p \to \mathbb{R} \text{ s.t. } (x,y) \to L(f(x),y) \in L^1(\pi) \}$ with respect to $d$, where $d$ denotes point-wise convergence, $L:\mathbb{R}^p \times \mathbb{R} \to \mathbb{R}$ be any positive, bounded, continuous function and let $\xrightarrow{m}$ denote set-wise convergence i.e $\mu_n(A) \to \mu(A) $ for every Borel set $A$, where $\mu_n,\mu \in \mathcal{P}(K)$. \\
\item \label{weak}$\xrightarrow{m}:= \rightharpoonup$ (weak convergence of measures),  $\mathcal{C}$ be compact such that $ \{ L(f(x),y) \}_{f \in \mathcal{C}} $ uniformly integrable with respect to $\{\pi_n \}$ and there exists $g $ such that $L(g(x),y) \in L_{\pi}^1$ such that $ f_n(x) \leq g(x)$ holds $\pi$-a.e.
\item \label{d_1ML} $\xrightarrow{m}:= \xrightarrow{d_1}$, $d$ point-wise convergence and $(x,y) \to L(f(x),y)$ uniformly Lipschitz and uniformly bounded for $f \in \mathcal{C}$ (compact metric space).
\item \label{tv} $\xrightarrow{m}:= \xrightarrow{TV}$, $L(x,y)$ is $d$-continuous on the first coordinate and uniformly bounded by some constant $M >0$ on a compact metric space $(\mathcal{C},d)$.
\end{enumerate}
\end{proposition}

\begin{proof}
In all of the cases above we only need to ensure a Fatou-like lemma (Definition \ref{flrs}).
\begin{enumerate}
    \item Here $\Gamma$-convergence must be thought as in Remark \ref{TopologyNotMetric}. This is a direct consequence of Fatou's lemma for varying measures (found in \cite{Royden} or \cite[Theorem 1.1]{FKL}). 
    \item See \cite[Theorem 2.2]{FKL}.
    \item The uniform Lipschitz condition gives 
    \begin{align*}
        & \bigg \lvert \int L(f_j(x),y) d\pi_n-d\pi(x) +  \bigg \rvert \int L(f_j(x),y)- L(f(x),y) d\pi(x,y)  \bigg \lvert \\
        & \leq d_1(\pi_n,\pi) + \bigg \rvert \int L(f_j(x),y)- L(f(x),y) d\pi(x,y)  \bigg \lvert 
    \end{align*}
    where the first term comes from Kantorovich-Rubinstein \cite[Particular Case 5.16]{VillaniOldAndNew} and the second one vanishes by dominated convergece. 
    \item In this case we don't only have the inequalities of definition \ref{flrs} but the limits coincide:
    \begin{align*}
      & \bigg \lvert   \int L(f_n(x),y) d \pi_n(x,y) - \int L(f(x),y) d\pi(x,y)  \bigg \rvert \\
      & \hspace{2 cm} \leq \bigg \lvert \int L(f_n(x),y) d(\pi_n-\pi) \bigg \rvert + \bigg \lvert \int L(f_n(x),y) - L(f(x),y) d\pi  \bigg \rvert \\
      & \hspace{2 cm} \leq M \lvert \lvert \pi_n - \pi \rvert \rvert_{TV} + \bigg \lvert \int L(f_n(x),y) - L(f(x),y) d\pi  \bigg \rvert 
    \end{align*}
    where the first one goes to zero by the assumption $\pi_n \xrightarrow{TV} \pi$ and the second one by the assumed $d$-continuity and dominated convergence.
\end{enumerate}
\end{proof}
The goal of this list is not to be exhaustive but to show the many different formulations that can be included in Definition \ref{flrs}. Notice that checking Definition \ref{flrs} involves only studying a two sided version of Fatou's Lemma that can be corroborated in every particular case. Once one establishes that the given ML problem of the form \eqref{pinProblem} and \eqref{piProblem} are indeed a full recovery system with $\{ \pi_n \},\pi, \mathcal{C},L$ one has ensured convergence of minimizers (which amounts to perfect approximation of the model). 
\begin{remark}
    Observe that the conditions imposed for $\mathcal{C}$ and $L$ on Proposition \ref{cases} case \ref{tv} are less restrictive than the ones on \ref{cases} case \ref{weak}. This is intuitively obvious as the total variation convergence is stronger than weak convergence. This means that ensuring a stronger convergence in measure is a degree of improvement for the ML-problem associated to fixed $\mathcal{C}$ and $L$. It is also evident that regularity conditions usually assumed in  ML-theory (like Lipschitz properties of $L$) yield strong approximations in most types of convergence $\xrightarrow{m}$, making this framework not only inclusive but rather general. 
\end{remark}

 One of the main advantages of measure pre-conditioning is the ability to change the training sample. It is common to use the empirical measure in non-parametric statistics, nevertheless the next section shows that the empirical measure is in general, not the best formulation for \eqref{pinProblem} as it may happen that the conditions for convergence hold for a different sequence of measures and not the sequence of empirical measures. We will see this is the case of Proposition \ref{cases} case \ref{tv}, where the sequence of empirical measures would not ensure subsequential convergence but a different sequence does, justifying completely the use of measure pre-conditioning as it improves the likelihood that the algorithm gives a reasonable final learnt agent. 
\begin{remark} (Compactness) \\
Stronger conditions like compactness of the underlying sets yield a more elegant theory. Many of the modes of convergence are equivalent under the assumption on compactness (see \cite{Billingsley} or \cite[Chapter 7]{Villani}). The assumption of compactness simplifyies most theorems as it will automatically bound sequences and so Definition \ref{flrs} is much easier to satisfy and verify which automatically yields: 
    
\end{remark}
\begin{proposition} If $\mathcal{C} \subseteq C(Y)$, $(x,y) \to L(x,y)$ is continuous and 
\[
\sup_{f \in C} \sup_{(x,y)} \lvert L(f(x), y) \rvert < \infty
\]
then \eqref{pinProblem} $\to$ \eqref{piProblem} in the $\mathcal{C}$ uniform topology, i.e. 
\[ 
 \argmin_{f \in \mathcal{C}} \mathbf{E}_{\pi} \left[ L(f(x),y) \right] \xrightarrow{\mathcal{C}}     \argmin_{f \in \mathcal{C}} \mathbf{E}_{\pi} \left[ L(f(x),y) \right].
\]
\end{proposition}
\subsubsection{No Empirical Probability Measure can Converge in the Total Variation Sense for all Distributions}
Towards studying when to measure pre-condition we realize that it is important to know what types of empirical measures converge and in which cases. In the seminal work  \cite{DevroyeGyorfi}, the authors proved the following theorem:
\begin{theorem}\label{empiricaltv} (No Empirical Probability Measure can Converge in the Total Variation Sense for all Distributions)   \\
Let $\{ \pi_n \}$ be a sequence of empirical distributions and $\delta > 0$, then there exists a proability measure $\pi$ such that \begin{equation*}
    \inf_n \sup_A \lvert \pi_n (A) - \pi(A) \rvert > \frac{1}{2} - \delta \text{ a.s. }
\end{equation*}
\end{theorem}
For a proof see \cite{DevroyeGyorfi}.\\
Theorem \ref{empiricaltv} tells us that the class of measures approximated in total variation norm by the empirical measure is not all measures. For different measures, other probability measures formed from data can converge in total variation but the empirical measure does not converge to all measures.
\begin{remark}\label{TVConvRemark} In \cite{DevroyeGyorfi} it is shown that the standard empirical measure does not converge in total variation sense to absolutely continuous limits. Hence, Theorem \ref{flrsmin} does not apply with Proposition \ref{cases} case \ref{tv} if we use the standard empirical measure. Nevertheless, as shown in \cite{Devroye}, the kernel-empirical measure given by \[
\pi_n = \frac{1}{h n} \sum K(f/h)
\]
does converge in total variation (see Definition \ref{KernelAC} below). Hence, Theorem \ref{flrsmin} via Proposition \ref{cases} case \ref{tv} applies to the sequence $\{\pi_n\}$ but not the sequence of standard empirical measures. This shows that the model solution for ML-program \ref{pinProblem} will converge to the best parametric $\mathcal{C}$-model. This argumentation \textbf{explains} why \textbf{standard techniques in Machine-Learning}, such as shifting and adding noise give \textbf{better results in practice}, as convergence is ensured by this system.
\end{remark}
\subsubsection{Example: Linear regression}
Let us consider $\pi \in \mathcal{P}_{ac}(\mathbb{R}^{2})$, we consider the linear regression problem with square-loss function with respect to target measure $\pi$: \begin{equation} \tag{TargetLR} \label{targetLR}
    \min_{(a,b) \in \mathbb{R}^2} \mathbf{E}_{\pi}[(Y-aX+b)^2].
\end{equation}
By differentiating with respect to $a,b$ from first order conditions we know that the solutions to \eqref{targetLR} are \begin{align}
    & a =  \frac{\displaystyle \int y \cdot x d\pi(x,y) - \int y d\pi(x,y) \int x d\pi(x,y)}{\displaystyle \int x^2 d\pi(x,y) - \left( \int x d\pi(x,y) \right)^2} \\
    & b = \int y d\pi(x,y)- \left(  \frac{\displaystyle \int y \cdot x d\pi(x,y) - \int y d\pi(x,y) \int x d\pi(x,y)}{\displaystyle \int x^2 d\pi(x,y) - \left( \int x d\pi(x,y) \right)^2}\right) \int x d\pi(x,y)
\end{align}
If we consider a sequence of measures $\pi_n$, obtained using the sample $(X_1,Y_1), \dots, (X_n,Y_n) $ then the linear regression problem with square-loss function with respect to approximating measure $\pi_n$ is
\begin{equation} \label{AppxLR} \tag{AppxLR}
    \min_{(a,b) \in \mathbb{R}^2} \mathbf{E}_{\pi_n}[(Y-aX+b)^2].
\end{equation}
The solution $(a_{\pi_n},b_{\pi_n})$ to \eqref{AppxLR} is given by 
\begin{align}
    & a_{\pi_n} =  \frac{\displaystyle \int y \cdot x d\pi_n(x,y) - \int y d\pi_n(x,y) \int x d\pi_n(x,y)}{\displaystyle \int x^2 d\pi_n(x,y) - \left( \int x d\pi_n(x,y) \right)^2} \\
    & b_{\pi_n} = \int y d\pi_n(x,y)- \left(  \frac{\displaystyle \int y \cdot x d\pi_n(x,y) - \int y d\pi_n(x,y) \int x d\pi_n(x,y)}{\displaystyle \int x^2 d\pi_n(x,y) - \left( \int x d\pi_n(x,y) \right)^2}\right) \int x d\pi_n(x,y)
\end{align}
If $\pi_n$ corresponds to the empirical measure, then rate of convergence of $a_{\pi_n}$ and $b_{\pi_n}$ have been widely studied. See \cite[Chapter 3]{MLBook} for example. We also know by Theorem \ref{empiricaltv} that $\pi_n \not \xrightarrow{TV} \pi$. By \cite[Section 2]{Devroye} we can find a sequence of measures (Parzen windows) $\{\tilde{\pi}_n\}$ such that $\tilde{\pi}_n \xrightarrow{TV} \pi$. \\
For simplicity, assume that \begin{equation*}
    \int x d\pi_n(x,y) = 0, \: \int x^2 d\pi_n (x,y) = 1,\:   \int x d\pi(x,y) = 0 \text{ and } \int x^2 d\pi (x,y) = 1.
\end{equation*}
With this assumption we immediately obtain the following bound: \begin{equation}
    \lvert a_{\pi_n} - a_{\pi} \rvert \leq \left(\sup_{(x,y) \in \spt(\pi_n) \cup \spt(\pi)} \lvert x \cdot y \rvert \right) \lvert \lvert \pi_n - \pi \rvert \rvert_{TV}.
\end{equation}
Which in the case where $\{\pi_n\},\pi$ are uniformly compactly supported yields \begin{equation} \label{uniformTVLinearBound}
    \lvert a_{\pi_n} - b_{\pi_n} \rvert \lesssim  \lvert \lvert \pi_n - \pi \rvert \rvert_{TV}.
\end{equation} 
Equation \eqref{uniformTVLinearBound} is a bound on the order of convergence on the coefficient of linear regression of \eqref{AppxLR} to that of \eqref{targetLR} which is not available in the case of the empirical measure, as indicated by Theorem \ref{empiricaltv}. The bound \eqref{uniformTVLinearBound} different to the usual order of convergence bounds for linear regression exemplifies the impact of measure pre-conditioning. Equation \eqref{uniformTVLinearBound} shows (uniform) stability of learning agents corresponding to the measure pre-conditioned problem, allowing us to use more tools than the standard ones.
\subsection{Measure pre-conditioning approaches} \label{Techniques}
Measure pre-conditioning approaches impose certain structures to the original data.  The idea is to analyze how does this structure impacts final outcomes of the modelling. In some way, this process resembles plain statistical inference.
\subsection{Background and Notation}
Let $\Omega \subset \mathbb{R}^n$ be fixed. We denote by $\mathcal{P}^p(\Omega)$ to be the set of probability measures with $p$-th finite moment. That is $\mathcal{P}^p(\Omega) = \{ \mu \in \mathcal{P}(\Omega): \int_{\Omega} \lvert x-x_0\rvert ^p d\mu < \infty, \text{ for some } x_0 \in \Omega\}$. We define the Wasserstein $p$ distance between  $\mu,\nu \in \mathcal{P}^p(\Omega)$ 
\begin{equation*}
    d_p(\mu,\nu) = \left( \inf_{\pi \in \Gamma(\mu,\nu)} \int_{\Omega \times \Omega} \lvert x-y\rvert^p d\pi(x,y) \right)^{1/p}
\end{equation*}
where $\Gamma(\mu,\nu)$ denotes the set of probability measures on $\Omega \times \Omega$ having first marginal $\mu$ and second marginal $\nu$. We say a map $T:\Omega_1 \to \Omega_2$ is a Monge map with respect to the cost function $c:\Omega_1 \times \Omega_2 \to \mathbb{R}$,  between Borel measures $\mu$ and $\nu$ whenever 
\begin{equation}\label{OT}
    T \in \argmin_{T \# \mu = \nu} \left\{ \int_{\Omega_1} c(x,T(x))d\mu(x) \right\}
\end{equation}
where $T \# \mu$ means that for every Borel set $A$, $\nu(A) = \mu(T^{-1}(A))$.
\section{Empirical measures and non-parametric estimation} 
\label{preconditionersandestimators}
In this section we discuss common non-parametric estimates and their relations to the structure of the ML-problems \eqref{pinProblem} and \eqref{piProblem}. We aim to explain how each measure can be used to pre-condition and the pros and cons coming with their use.
\subsection{Non-exhausting list of non-parametric estimation techniques} \label{definitions}
\begin{definition} \label{empirical} (Empirical measure) \\
Given $X_1, \dots, X_n$ we define the standard empirical measure as the number of successes on the $n$ occurrences: 
\[ \mu_n (A) = \frac{1}{n} \sum_{k=1}^{n} \delta_{X_k}(A).
\]
\end{definition}
\begin{definition} \label{histogram} (Histogram) \\
Given $X_1, \dots, X_n$ we define the histogram measure associated to the sets $B_1,\dots, B_m$
\begin{equation*}
    \mu_n (A) = \frac{1}{n} \sum_{k = 1}^n \sum_{l = 1}^m \frac{1}{\rho(B_l)} \delta_{X_k}(A \cap B_l).
\end{equation*}
where $\rho$ is a probability measure (usually taken to be normalized Lebesgue).
\end{definition}
\begin{definition} (Kernel estimation via Parzen windows) \label{KernelAC} \\
Given $X_1, \dots, X_n$, we define the $n$-th density estimation with kernel $K$ via
\[
f_{\pi_n}(x) = \frac{1}{nH_n} \sum_{i=1}^n K\left( \frac{x-X_i}{H_n} \right) 
\]
where $K$ is fixed and $\{H_n\}$ is any sequence of random variables, that (may) depend on the sample $X_1, \dots, X_n$ that satisfy that $H_n \to 0$ almost surely and $ nH_n \to \infty$ almost surely.
\end{definition}
The idea of this formulation of the kernel estimation comes from \cite{Parszen} and \cite{Rosenblatt} and it is fully justified by Theorem \ref{convDensities}.
\subsubsection{Wasserstein 2-Barycenter}
\begin{definition}(Wasserstein Barycenter) \label{WassBarycenter}\\
Given a sample $X_1, X_2, \dots, X_n$ random variables in $\mathbb{R}^p$ we define the $2$-Wasserstein Barycenter of the sample (also called Frechet mean) as any probability measure satisfying
\begin{equation}
    \mu^* \in \argmin_{\rho \in \mathcal{P}^2(\mathbb{R}^p)}\left\{ \sum_{k=1}^n d_2(\rho,\delta_{X_k})^2\right\} 
\end{equation}
where $\delta_{X_k}$ denotes the unit mass at $X_k$.
\end{definition}
\begin{remark} Note that $ \rho \to d_2(\cdot, \nu)^2$ is lower-semicontinuous for all $\nu$ and so Wasserstein Barycenters exist. In general, Wassertein barycenters with respect to random Dirac measures are not unique. If instead, one of the deltas is replaced by an absolutely continuous measure, uniqueness can be shown. We don't do this replacement in this document, instead we study the entropic regularization of the minimization problem in Definition \ref{EntroBary}.
\end{remark}
The theory of Wasserstein Barycenters has recently received attention from several fields of applied mathematics, see for example \cite{Invitation} for a more complete theory.
\begin{remark}
The barycenter can be defined given any distance function $d: \mathcal{P}(\mathbb{R}^p) \times \mathcal{P}(\mathbb{R}^p) \to \mathbb{R}$ and a sample $(X_1,\dots,X_n)$ the $d$-barycenter is any probability measure $\mu $ satisfying
\begin{equation}
    \mu^* \in \argmin_{\rho} \left\{ \frac{1}{n} \sum_{k=1}^n d(\rho,\delta_{X_k} )\right\}
\end{equation}
where the infimum is taken over all probability measures on $\mathbb{R}^p.$. We have only chosen the Wasserstein 2-distance as we aim to focus on Domain Adaptation.
\end{remark}
\begin{remark}
    It is important to notice that efficient algorithms to compute Wasserstein Barycenters have recently been developed (see \cite{CuturiBarycenter}) in the case of empirical measures. This efficient computability is essential for the applications we have in mind. 
\end{remark}
\subsubsection{Uniform convex hull}
\begin{definition}\label{ConvexHull} (Convex Hull)\\
The convex hull of a set $B \subseteq \mathbb{R}^p$ is defined to be the smallest convex set on which $B$ is contained, equivalently \begin{equation*}
\conv(B)  = \bigcap_{\substack{C \text{ convex } \\ B \subseteq C}} C.  
\end{equation*}
We define the uniform convex hull of the sample $(X_1, X_2,\dots X_n)$ to be the uniform measure on the convex hull of $\{X_1, X_2, \dots, X_n\}$, i.e.
\begin{equation}
    \mu_{conv} = \frac{\mathcal{L}^p \mid_{c}}{\mathcal{L}^p(\conv(\{X_1,X_2,\dots,X_n\}))}
\end{equation}
where $\mathcal{L}^p$ denotes the Lebesgue measure in $\mathbb{R}^p$.
\end{definition}
\begin{remark}
    Note that $\mu_{conv}$ is the restriction of the Lebesgue measure to the convex hull of the sample so it's support is automatically convex. This particular property could be significant for future applications as the theory of convex optimization unlocks several numerical techniques. Evidently, it's support also includes all points of the sample. 
Note that Definition \ref{ConvexHull} always gives a well defined measure. 
\end{remark}
\subsubsection{Entropically regularized barycenter}
\begin{definition}\label{EntroBary}
Given a sample $X_1, X_2, \dots, X_n$ random variables in
$\mathbb{R}^p$ and a reference probability measure $\nu$ we define the $\nu$-entropically regularized $2$-Wasserstein Barycenter of the sample as any probability measure satisfying
\begin{equation} \label{minentrobary}
    \mu^* \in \argmin_{\rho \in \mathcal{P}^2(\mathbb{R}^p)}\left\{ \frac{1}{n} \sum_{k=1}^n d_2(\rho,\delta_{X_k})^2 + \ent(\rho \mid \nu) \right\} 
\end{equation}
where $\delta_{X_k}$ denotes the unit mass at $X_k$ and $\ent(\mu \mid \nu)$ denotes the relative entropy of $\rho$ with respect to $\nu$ given by 
\begin{equation}
    \ent(\rho \mid \nu) = \int \log\left( \frac{d\rho}{d\nu} \right) d\nu
\end{equation}
whenever $\rho \ll \nu$ and $\ent(\rho \mid \nu) = \infty$ otherwise.
\end{definition}
\begin{remark}
    If $\nu \ll \mathcal{L}^p$, the functional to minimize is lower semi-continuous and with the addition of entropy a unique absolutely continuous minimizer of \eqref{minentrobary}.
\end{remark}
\subsubsection{Class-regularized barycenter}
Motivated from the work of \cite{CourtyFlamary} we can also think of measure pre-conditioning in terms of pre-established class based groups. The idea behind the next definition is that elements in the same class may be very similar while elements from different classes could be very different from each other.
\begin{definition} (Class barycenter) \label{ClassBary} \\
Given a sample $X_1, X_2, \dots, X_n$ random variables in
$\mathbb{R}^p$ suppose that each $X_i$ belongs to one and only one of a finite collection of classes $\{ C_l\}_{l = 1}^m$, then we can define the class-based barycenter to be any measure $\mu$ satisfying
\begin{equation}
    \mu^* \in \argmin_{\mu \in \mathcal{P}^2(\mathbb{R}^p)}\left\{ \frac{1}{m} \sum_{k=1}^m d_2(\rho,\nu_k)^2 + \ent(\rho \mid \nu) \right\} 
\end{equation}
where $\nu_k$ is a measure determined only from class $C_k$. For example, one would obtain a barycenter of barycenters if one were to choose $\nu_k$ to be the $2$-Wasserstein barycenter of $\{ X_i: X_i \in C_k \}$.
\end{definition}
\subsubsection{MMD-regularized Conditional measures} 
\begin{definition}
\label{MMDr}
Given a characteristic kernel function $k$ (see \cite{Sriperumbudur} for details), define the maximum mean discrepancy between $\mu,\nu$ with respect to $k$ via \begin{equation*}
    \mmd_k(\mu,\nu) = \mathbf{E}_{\mu \times \mu}[K(X,\Tilde{X})] + \mathbf{E}_{\nu \times \nu}[k(Y,Y)] - 2 \mathbf{E}_{\mu \times \nu} [k(X,Y)]
\end{equation*}
The empirical optimal transference plan between conditional distributions for a given lower-semicontinuous cost function $c$, denoted $\pi_n^{*,c}$ is defined in \cite{EOTBC} via the minimization over $\Gamma(\mu.\nu)$ of the following functional: \begin{equation}\label{EOTBCD}
\int c(x,y) d\pi + \lambda_1 \frac{1}{n} \sum_{i=1}^n \mmd_k^2( Proj^1 \# \pi , \delta_{Y_i}) + \sum_{i=1}^n \mmd_k^2( Proj^1 \# \pi' , \delta_{Y'_i}).
\end{equation}
\end{definition}
Existence and uniqueness depends on the cost function and usual conditions (smoothness and twist) are required, see \cite{Sriperumbudur} for details.
\subsection{Some properties of the measure pre-conditioners}
\begin{proposition}\label{Baryac} When they exist, the measures from definitions \ref{EntroBary} and \ref{ClassBary} are absolutely continuous with respect to $\nu$.
\end{proposition}
\begin{proof}
    By definition,  $Ent(\rho \lvert \nu) =  \infty $ if $\rho \not \ll v$, because $\nu$ is always feasible, the functional is not infinity and hence the minimizer is a.c. with respect to $\nu$.
\end{proof}
\begin{corollary}\label{Barycor} If $ \nu = \mathcal{L}^p$ in Definitions \ref{EntroBary} or \ref{ClassBary}, the minimizer has a density (w.r.t. Lebesgue). 
\end{corollary}
Although the proof is simple, the importance of Proposition \ref{Baryac} and Corollary \ref{Barycor} is fundamental for practice. If we can estimate the density, we can use it to improve the convergence of algorithms by numerical methods. See for example \cite{CuturiPeyre} where the entropic regularization allows a closed (and very simple) form of the density which then yields a dual-descent algorithm. Knowing explicitly the density allows us to find minimizers of Problem \ref{pinProblem} via formulae and so we can focus our attention on estimating numerically these minimizers without carrying a second numerical error. 
\subsection{Optimality (Euler-Lagrange)}
Most of the measure pre-conditioners defined on section \ref{Techniques} require the minimization of a functional. Let $\Omega \subseteq \mathbb{R}^p$, in this section we study the first order conditions for minimization in $(\mathcal{P}_2(\Omega),d_2) $ which can be found in \cite[Theorem 7.20]{Santambrogio}.
\begin{definition}(First variation of a functional in $\mathcal{P}(\Omega)$) \\ 
Let $F$ be a functional $F: \mathcal{P}_2(\Omega) \to \mathbb{R} $, let $\rho \in \mathcal{P}_2(\Omega)$ be fixed  and $\epsilon > 0$, for any $\Tilde{\rho} \in \mathcal{P}^2_{ac} \cap L^{\infty}(\Omega)$, define $\nu = \Tilde{\rho} - \rho $, we say that $\frac{\delta F}{\delta \rho} (\rho)$ is the first variation of $F$ evaluated at $\rho$ if 
\begin{equation*}
\frac{d}{d \epsilon} \bigg \lvert_{\epsilon = 0} F(\rho + \epsilon \nu) = \int \frac{\delta F}{\delta \rho}(\rho) d\nu.
\end{equation*}
\end{definition}
\begin{theorem} (Optimality criteria)
\label{firstvarML} \\
For a functional $F: \mathcal{P}_2(\Omega) \to \mathbb{R} $ suppose that $\mu \in \argmin_{\nu \in \mathcal{P}_2(\Omega)}  F(\nu)$. Assume that for every $\epsilon > 0$ and for every $\rho$ absolutely continuous with $L^{\infty}(M)$ density 
$$ F((1-\epsilon) \mu + \epsilon \rho) < \infty $$
let $c:= essinf\left\{\frac{\delta F}{\delta \rho} (\mu) \right\}$. If  $\frac{\delta F}{\delta \rho} (\mu)$ is continuous,
\begin{equation} 
\frac{\delta F}{\delta \rho} (\mu)(x) \geq c \: \: \forall x \in M,
\end{equation}
\begin{equation}
     \frac{\delta F}{\delta \rho} (\mu)(x) = c \: \: \forall x \in supp(\mu).
\end{equation}
\end{theorem}
The proof can be found as Theorem 7.20 in \cite{Santambrogio}. \\
Just as in the remark after Corollary \ref{Barycor}, the main use of this tool is to focus the algorithmic implementation towards the computation of the first variation of the functional it minimizes. 
\subsection{Convergence}
The objective of the reformulation of the general ML-problem in terms of Problem \ref{pinProblem} and \ref{piProblem} is that we can adapt every stage of the learning process by using a measure estimation that fits the problem better. In order for us to know that we can recover the ML-problem in this process we need to know the types of convergence on which the sequences of measures formulated with the data converge to the underlying distribution. Many theorems and specific cases on density estimation have been studied, we recollect some of them here in terms of the definitions of section \ref{definitions}.
\subsubsection{Convergence of density estimations} \label{convergenceSection}
Observe that Theorem \ref{flrs} and Proposition \ref{cases} allow different systems of convergence, i.e. depending on the `strength' of the type of convergence $\xrightarrow{m}$ of the probability measures, different requirements on $\mathcal{C},d,L$ are needed. In this section we give a non-exhaustive list of modes of convergence for density estimation and the sequences in Section \ref{definitions} that can be used as measure pre-conditioners. In this section one should notice that every type of convergence should be coupled with hypothesis that ensure the system is a full learner recovery system (Definition \ref{flrs}).
\begin{theorem}\label{GlivenkoCantelli} (Glivenko Cantelli in $\mathbb{R}$) \\
Let $\mu$ be any probability measure on $\mathbb{R}$ and $\mu_n$ be the standard empirical measure (Definition \ref{empirical}), if $F(t) = \mu((-\infty,t])$ and $F_n(t) = \mu_n((-\infty,t])$ then $F_n \to F$ uniformly on $\mathbb{R}$  as $n \to \infty$
\end{theorem}
This theorem is well-known see for example \cite[Theorem 7.4]{Durrett} or \cite[Theormem 11.4.2.]{DudleyProb}. By account's of Donsker's theorem one can get the following improvement: 
\begin{proposition} \label{RateGlivenko}(Rate of convergence for continuous $F$) \\
If $\mu$ is a law on $\mathbb{R}$ for which $F$ is continuous, the order of convergence of Theorem \ref{GlivenkoCantelli} satisfies 
\begin{equation}
    n^{1/2} \sup_{t} \lvert F_n(t) - F(t) \rvert  \rightharpoonup \max_{0 \leq s \leq 1} \lvert B_s - sB_1 \rvert
\end{equation}
    where $ \{B_s\}$ is a Brownian motion, i.e. the rate of convergence approaches the law of the absolute value of a Brownian bridge on $[0,1]$ and so it's law can be computed explicitly:
    \begin{equation}
        P_0\left( \sup_{0 \leq s \leq 1} \lvert B_s - s B_1 \rvert < b \right) = \sum_{m =-\infty}^{\infty} (-1)^m e^{-2m^2 b^2}
    \end{equation}
\end{proposition}
See \cite[Theorem 8.10]{Durrett} and the following proposition for the explicit formula of it's law.
\begin{remark}
    The theorem presented here as Theorem \ref{GlivenkoCantelli} is just a specific version. In general, one refers to any theorem of this type as ``a Glivenko-Cantelli type theorem'' see for example \cite{DudleyProb}.
\end{remark}
\begin{theorem} \label{Varadarajan} (Varadarajan) \\ If $\pi$ is any probability measure on $X \times Y$ and $X \times Y$ is a separable metric space then the standard empirical measures (Definition \ref{empirical}) for $(X,Y)$ converge weakly in probability to $\pi$. 
\end{theorem}
For a proof see \cite[11.4.1]{DudleyProb}. It is important to notice that the convergence is almost surely. In some cases, like the case of real numbers, the convergence can be upgraded.
\begin{remark}
    Notice that from Theorem \ref{GlivenkoCantelli} one can infer the convergence of the Histogram (Definition \ref{histogram}) weakly in $\mathbb{R}^p$.
\end{remark}
\begin{theorem}\label{convDensities} (Devroye) \\
If $H_n^2 n \to \infty$ and $ \mu \ll Leb$, the empirical density estimate of Definition \ref{KernelAC} converges  uniformly in measure to $\mu$, i.e. for every $\epsilon > 0$, 
\begin{equation}
    P\left(\left\{ \omega: \sup_{x \in \mathbb{R}} \lvert f_n(x,\omega) - f(x) \rvert < \epsilon \right\} \right)\xrightarrow{n \to \infty} 1. 
\end{equation}
\end{theorem}
For a proof see \cite{Rosenblatt}.\\
The following theorem is a specific case of the much more general convergence of Barycenters proved in \cite{CGP}, in the paper the authors prove the $d_p$-convergence in metric measure spaces satisfying a positive curvature condition.  \\
\begin{proposition} \label{Baryconverges}(Barycenters $d_p$ converge) \\
If $\mu$ has compact support and $\mu_n$ is the a $p$-Wasserstein Barycenter of Definition \ref{WassBarycenter}, then $\mu_n \xrightarrow{d_p} \mu$ as $ n \to \infty$.
\end{proposition}
For a proof see \cite{CGP} and apply it to the simple case where $(R^p,\lvert \cdot \rvert, \mu)$ is given as the initial measure space.
In \cite{EOTBC} the following proposition was shown: 
\begin{proposition} (Total variation) \label{totalvariationproposition} \\
The $\mmd_k$ minimizer of Definition \ref{MMDr}, $\pi_n^{\mmd_k}$ converges in total variation norm to the solution $\pi^*$ of unrestricted transport with respect to $c$ (Definition \ref{OT}), i.e. \begin{equation*}
    \pi_n^{\mmd_k} \xrightarrow{\lvert \lvert \cdot \rvert \rvert_{TV}} \pi^*.
\end{equation*}
\end{proposition}
See \cite[Theorem 1]{EOTBC}.
\subsubsection{Convergence and full learner recovery systems}
In the previous section \ref{convergenceSection} we have listed several convergence results for different types of empirical measures. Empirical measures encompass our understanding of the sample. Theorems \ref{GlivenkoCantelli}, \ref{Varadarajan} and \ref{convDensities}, Propositions \ref{RateGlivenko}, \ref{Baryconverges} and \ref{totalvariationproposition} need to be coupled with regularity properties of $L$ and the underlying class of functions $\mathcal{C}$ as in Proposition \ref{cases}. This list shows that given an underlying model, it's intrinsic features will determine the type of measure pre-conditioners needed to ensure convergence on the specific convergence mode that the limiting measure admits. \\ For example, Proposition \ref{totalvariationproposition} involves convergence in Total Variation norm from which one can infer that the measure pre-conditioning of Definition of \ref{MMDr} applies for a $d$-continuous (in the first coordinate) loss function $L$ as in Proposition \ref{cases}. \ref{tv}. In contrast, Theorem  \ref{empiricaltv} shows that the empirical (uniform) measure is not well-suited for every limiting distribution and so in the case of a continuous density, preconditioning by \ref{MMDr} is proved to have better results (theoretically) than the empirical measure.
\subsubsection{Estimating the marginal instead}
In the discussion of density estimation (Section \ref{definitions}) we haven't done any specific distinction on the particular form the data for Problems \ref{pinProblem} and \ref{piProblem}. Definitions \ref{empirical}-\ref{MMDr} work for all kinds of data. In the particular case of the ML Problems \ref{piProblem} and \ref{pinProblem}, our objective is to model in the class $\mathcal{C}$ the dependence of $Y$ on $X$ penalized by the loss function $L$. We aim to study how good (with respect to $L$) a $\mathcal{C}$-model $f(X)$ approximates $Y$. In this context the distribution $\pi$ refers to that of $(X,Y)$. Measure pre-conditioning amounts to approximating $\pi$ using the sample in a way that benefits computations. We note that this gives rise to two different approaches: \begin{enumerate}[label=(\alph*)]
    \item \label{approachpi} We can estimate $\pi$ directly via $\pi_n$ according to definitions \ref{empirical} -\ref{MMDr}. 
    \item \label{approachmu} We can make assumptions on the conditional distribution of $Y \lvert X$ and then use  definitions \ref{empirical}-\ref{MMDr} for approximations on the $X$-marginal of $\pi$.
\end{enumerate}
Most of the study of this document has focused on approach \ref{approachpi}. Let us give an example of the approach \ref{approachmu} to show it's interaction with measure pre-conditioning. \begin{theorem} \label{desintegrateuniform}
    Assume that $ Y \lvert X = x \sim \nu_x$ and that we have estimated $\nu_x$ via $\nu_n^x$ such that $\nu_n^x \xrightarrow{d_p} \nu_x$ uniformly on $x$, i.e. given $\epsilon$ there exists $N > 0$ such that for every $n \geq N$ \[
     d_p(\nu_x, \nu_n^x) < \epsilon \text{ for every } x
    \]
    assume also that $\mu_n \xrightarrow{d_p} \mu$, and $L: \mathbb{R}^p \times \mathbb{R} \to \mathbb{R}$ is continuous. Let $f \in \mathcal{C}$ and assume that there exists $ g \in L^1(\mu)$ such that 
    \begin{equation*}
        \bigg \lvert \int L(f(x),y) d\nu_n^x (y) \bigg \rvert \leq g(y). 
    \end{equation*}
and that $ y \to \int L(f(x),y) d\nu_n^x (y)$ is continuous and bounded, then
\begin{equation*}
    \int \int L(f(x),y) d\nu_n^x(y) d\mu_n (x) \xrightarrow{n \to \infty} \mathbf{E}_{\pi}[L(f(X),Y)].
\end{equation*}
\end{theorem}
\begin{proof}
    The proof is a direct consequence of dominated convergence applied twice, observe that
    \begin{align*}
        & \int \int L(f(x),y) d\nu_n^x d\mu_n (x) - \int L(f(x),y) d\pi = \\
        & \int \int L(f(x),y) d\nu_n^x d\mu_n (x) - \int \int L(f(x),y) d\nu_n^x d\mu (x) + \\
        & \hspace{3 cm} \int \int L(f(x),y) d\nu_n^x d\mu (x) - \int L(f(x),y) d\pi.
    \end{align*}
    the first term goes to zero if we ensure $ y \to  \int L(f(x),y) d\nu_n^x (y) $ is continuous and bounded, the second term goes to zero by dominated convergence (using $\mu$ as reference measure).
\end{proof}
\begin{remark} \label{remarkYgivenX}(On the general approach and the restrictiveness of the hypothesis on Theorem \ref{desintegrateuniform}).\\
Theorem \ref{desintegrateuniform} is \textit{ only one example} of the multiple approaches one can use to estimate $\pi$ from $\mu_n$ and assumptions on $Y \lvert X$, even though the hypothesis of Theorem \ref{desintegrateuniform} are very difficult to meet in practice, it is presented here to illustrate the general idea. \\Measure pre-conditioning on the marginals $\nu^x$ allows the modeller to include the specific features of each data class. It is clear the many lines of investigations one can explore to get similar results (with less restrictive hypothesis), we choose not to develop any further and leave it for future research. 
    
\end{remark}
\subsection{The recipe: How to choose a measure and how to implement the algorithm} \label{recipe}
The general approach for this document is to put in a single, standard, theoretical background many ideas that have come to light in ML-research. Namely, ML-reaserchers have realized that their algorithms improve in performance or convergence properties after a small ``tweak'' to either data or the loss function occurs. Stability of ML-algorithms has been widely known and is one of the main focus of ML-research. The idea of measure pre-conditioning is that the standard empirical distribution, though it may contain all the possible information in terms of inference (except for order) may not be well adapted to the specific problem one aims to minimize. It is well-known for example that if the functional to be minimized is convex, algorithms used for minimization can take advantage of convexity. This encourages the solver to find an empirical estimation from definitions \ref{empirical}-\ref{MMDr} that makes their functional convex. Finding such a measure is what we call \textbf{pre-conditioning}, if the preconditioning satisfies any of the assumptions of Proposition \ref{cases} then one is ensured to have a full learner recovery system and hence have not lost anything on the process while achieving improved performance. One could instead use a pre-conditioner based on many reasons (such as having a specific algorithm to compute already at hand for example), this work explains how as soon as a condition like Proposition \ref{cases} is satisfied, one will end up with the same classifier/regressor.
\section{The problem of Domain Adaptation and the impact of measure pre-conditioning} \label{DomainAdaptationSection}
Domain Adaptation (DA) is a sub-problem of transfer learning on which one aims to infer the parameters for a new learning agent in terms of an agent that learn in similar data. \\
Many of the DA adaptation formulations are well-suited for Optimal Transport (OT), our framework of Problems \ref{pinProblem} and \ref{piProblem} was motivated at first by the recent research in optimal transportation in Machine Learning (see \cite{Meta}, \cite{COOT}, \cite{Courty}, \cite{CourtyFlamary}) and so in this section we explore the implications of measure-preconditioning in the specific case of domain adaptation problems related to optimal transportation and the recent research in the area (see \cite{Courty}, \cite{CourtyFlamary} and references therein for a more complete exposition of the use of optimal transportation in machine learning).
\begin{problem} \label{DATransfer} (General domain adaptation problem) \\
Suppose that we have a sample $(X^s_1, X^s_2, \dots, X^s_n)$ of features together with the a sample of the dependent variable $(Y^s_1,Y^s_2, \dots , Y_n^s)$ and we use the learning agent to minimize a loss function $L:\mathbb{R}^p \times \mathbb{R} \to \mathbb{R}$ among a class of functions $\mathcal{C}$. The learning problem is to obtain the best possible parametric function $f$, among the class $\mathcal{C}$ explaining the data, i.e.
\begin{equation}\label{Learners}
    \min_{f \in \mathcal{C}} \left\{ \sum_{k=1}^n \mathbf{E}[L(f(X_i^s),Y_i^s)]\right\}
\end{equation}
If $f_s^*$ realizes the minimum in \eqref{Learners}, we say that it is the learnt agent or that $f_s^*$ correspond to the learnt parameters.\\
Now suppose we have another sample $(X^T_1, X_2^T, \dots, X_{n_2}^T)$ which we believe is similar in some features to the original sample. The domain adaptation problem is: How much can one learn from the previous learning? 
That is, how can we transfer the learning from  the source domain to target domain?. 
\end{problem}
The research field which attempts to answer Problem \ref{DATransfer} is known as Domain adaptation for transfer learning. For a general introduction and approach see \cite{Courty}, \cite{MLBook} and references therein. \\
The problem of domain adaptation \ref{DATransfer} is different to Problems \ref{piProblem}, \ref{pinProblem} as it aims to transfer the statistical knowledge obtained by a minimization on source-domain to a minimization on the target-domain. The formulation of Problem on \eqref{Learners} has the implicit assumption of the empirical distribution being imposed at level $n$.\\
In this section we aim to explain how measure pre-conditioners as defined in section \ref{definitions} can be used in the field of DA for transfer learning. 
\begin{problem} \label{DATransferVLC} (Domain Adaptation and transfer learning with varying losses and classes)  \\
    Suppose that we have a sample $(X^s_1, X^s_2, \dots, X^s_n)$ of features together with the a sample of the dependent variable $(Y^s_1,Y^s_2, \dots , Y_n^s)$ and we use the learning agent to minimize a loss function $L_s:\mathbb{R}^p \times \mathbb{R} \to \mathbb{R}$ among a class of functions $\mathcal{C_s}$. The learning problem is to obtain the best possible parametric function $f$, among the class $\mathcal{C}$ explaining the data, i.e.
\begin{equation} \label{pinDAProblem}
    \min_{f \in \mathcal{C_s}} \left\{  \mathbf{E}_{\pi_n} [L_s(f(X^s),Y^s)]\right\}
\end{equation}
and compare it with the perfect learner on target domain with class $\mathcal{C}_t$ and loss function $L_t:\mathbb{R}^p \times \mathbb{R} \to \mathbb{R} $:
\begin{equation} \label{pitDAProblem}
    \min_{f \in \mathcal{C}_t} \left\{  \mathbf{E}_{\pi^t} [L_t(f(X^t),Y^t)]\right\}
\end{equation}
If $f^*$ denotes the minimizing argument for \eqref{pinDAProblem}, the Domain Adaptation problem is: How can we use $f^*$ to obtain good estimates for \eqref{pitDAProblem}? \\ What is the structure of such agent? \\ How does it compare to the actual minimizer of \eqref{pitDAProblem}? 
\end{problem}
Suppose that every $X_i^s \sim \mu_s$ and $X_i^t \sim \mu_t$, under ``similarity assumptions''  on $\mu_s$ and $\mu_t$, one expects to be able to transfer learning to some accuracy. Of course ``similarity assumptions'' depends on the context of the ML-task in hand. \\ For example, two measures might be considered similar in a classification problem that may not be considered similar in a generative model.
In the same fashion, suppose that $\mu_s$ and $\mu_t$ satisfy that there exists a solution, $T$, for Problem \eqref{OT} with a given cost function $c:\mathbb{R}^p \times \mathbb{R}^p \to \mathbb{R}$. A good candidate for a new learnt agent can be immediately obtained via $f^* \circ T^{-1}$. As seen in \cite{CourtyFlamary}, the error made by this agent relative to the total error obtained from training an agent from scratch can be controlled as soon as $\mu_s$ and $\mu_t$ are $d_2$-close and $\mathcal{C}$ is rich enough. In the field of Domain Adpatation (DA) usually at least one of the following assumptions is made:
\begin{assumption}(Conditional structure of learning task) \label{ConditionalAssumption} \\
In the context of Problem \ref{DATransferVLC}, if $(X^s,Y^s)$ is the source variable and $(X^t,Y^t)$ the target variable, it is common to ask that
\begin{equation}
    (Y^s_i \mid X^s_i) \sim (Y_i^t \mid X_i^t), 
\end{equation}
where $Y \mid X$ denotes the random variable whose law is the regular conditional probability of $Y$ given $X$.  
\end{assumption}
This assumption means that the probabilistic structure of the dependence of $Y$ on $X$ is the same in both domains. We understand this assumption as a strong hypothesis of similarity in the modellings. 
\begin{assumption} (Identical dependence) \label{identical} \\ In the context of Problem \ref{DATransferVLC}, if $(X^s,Y^s)$ is the source variable and $(X^t,Y^t)$ the target variable, it is common to ask that
\begin{equation*}
(X_s,Y_s) \sim (X_t,Y_t)
\end{equation*}
\end{assumption}
The identical dependence assumption has been used extensively but is in general not a good idea to pre-impose. The identical assumption implies that any sample of the source domain can be considered a sample of the target domain so if $L_s = L_t$ and $\mathcal{C}_s = \mathcal{C}_t$ then the learning transfer is perfect as we can identify the source data as target data in the empirical destimation of $\pi_s = \pi_t$
The following assumption can be found in recent papers in DA-ML, see \cite{Courty} for example.
\begin{assumption} \label{TransportIndep} ($c$-optimal map) \\ There exists an optimal transport map (with respect to a cost function $c: \mathbb{R}^p \times \mathbb{R}^p \to \mathbb{R}$) $T_c$ as in \eqref{OT} that satisfies 
\[(X_i^s,Y_i^s) \sim (T_c(X_i^s),Y_i^{t}). \]
\end{assumption}
\begin{remark}
    Though it is straightforward to use Assumption \ref{TransportIndep} (postulated in \cite{Courty}) in the context of optimal transportation, it is of significant importance to understand the necessary conditions that yield this assumption.
\end{remark}
\begin{remark} Note that these assumptions and the framework of DA is closely related to the line of investigation proposed in Remark \ref{remarkYgivenX} below.
\end{remark} 
\subsection{General Idea in the non-linear case}
Domain Adaptation should be used when the target and source measures are believed to be similar. If the source measure satisfies the assumptions of Brenier's Theorem (see \cite[Theorem 2.32]{Villani}) and the loss function is quadratic (or strictly convex function of quadratic distance) the optimal transport map $T$ transporting $\mu_s$ onto $\mu_t$ can be used as an learning agent on the target domain. We do this by first mapping onto the source domain using the optimal transport map and only then evaluating the agent that has learnt paramters, i.e. define $f_{ad}$ a candidate for the minimization of loss for learning agents in the target domain by
\begin{equation*}
    f_{ad} = f^* (T^{-1}).
\end{equation*}
The work in \cite{Courty} shows a convergence for this agent under Assumption \ref{ConditionalAssumption}.
\subsection{Main question: What cost should we impose? } \label{main}
Note that Assumption \ref{TransportIndep} is an existence condition. If there exists a cost function for Assumption \ref{TransportIndep} one would need to check that it satisfies the conditions for existence and uniqueness of optimal transport maps like regularity and the twist condition (see \cite{McCannFiveLectures}, \cite{Villani},\cite{Santambrogio}). \\
In the general approach for DA using transfer learning via optimal transport in the framework of Problem \ref{DATransferVLC}, two problems seem to arise more often in practice: \begin{enumerate}[\textbf{P}.i)]
    \item \label{p1}When the conditions of the trainings are fixed and not to be chosen: study a learnt agent when $L_1,L_2, \mathcal{C}_1,\mathcal{C}_2$ are given and fixed.
    \item \label{p2} When we are able to choose $L_1,\mathcal{C}_1$ with the goal of maximizing (in any way) the transfer learning for a given loss function $L_2$ and class $\mathcal{C}_2$.
\end{enumerate} 
\subsection{A measure of transferrability}
In Problems \ref{pinDAProblem} and \ref{pinProblem}, we start under a similarity assumption on the source measures. This follows an intuitive statement: in order to be able to transfer any learning, the original measures should share some features. We can't expect to transfer any learning if the problems have nothing in common. \\
We may expect to transfer the learning (classifier) differentiating between dogs and cats to a new agent aiming to differentiate wolves and lions. In this case the distribution of dogs and cats is believed to be similar to that of wolves and lions. \\ How much could we transfer? Could we guess beforehand how much learning we can transfer? \\
As a thought experiment, let us study a way to measure the transfer of learning. There are many ways to measure transferability, see \cite{CourtyFlamary}, \cite{Courty} or references therein. We propose another one, assume that $\pi^s$ and $\pi^t$ are as in Problem \ref{DATransfer}, let $h:\mathbb{R} \to \mathbb{R}$ be any \textit{strictly} convex function with $h(0) = 0$, set \begin{equation}\label{measureTransfer}
d_h(\pi^s, \pi^t) = \inf_{\Pi \in \Pi(\pi^s, \pi^t)}  \int h(L_1(f_1(x_1),y_1) - L_2(f_2(x_2),y_2)) d\Pi((x_1,y_1),(x_2,y_2)). 
\end{equation}
where $f_1$ is the solution for the $\mathcal{C}_1,L_1$-source problem and $f_2$ the corresponding solution for the $\mathcal{C}_2,L_2$-target problem. Evidently, a-priori the value of $d_h(\pi^s, \pi^t)$ can not be computed as $f_1,f_2$ are unknown and the value of \eqref{measureTransfer} depends on the choice of models $(\mathcal{C}_1,L_1)$ and $(\mathcal{C}_2,L_2)$. We claim \eqref{measureTransfer} is a reasonable way to measure transfer depending on $\mathcal{C}_1, L_1, \mathcal{C}_2,L_2$, in the sense that the closest $d_h$ is to $0$ the more likely it is that a learnt agent for the $L_1$ problem with source data $(X^s,Y^s)$ would perform decently in the $L_2$ problem with data $(X^t,Y^t)$. This is to be expected as it may be reasonable to transfer the learnt agent for certain loss functions but not with all of them. \\
Even though $f_1$ and $f_2$ are unknown, in some cases some estimates can be obtained. To the knowledge of the author no measure of transferrability of the form \eqref{measureTransfer} has been studied which points to a promising line of investigation.
\subsection{Problem 1} \label{costP1}
Let us first address problem \ref{p1} where all the conditions ($\mathcal{C}_1,\mathcal{C}_2,L_1,L_2$) are fixed and we aim to measure the efficiency of a solution to \eqref{piProblem} and $\eqref{pinProblem}$.

\subsubsection{Measure pre-conditioning in the conditional average guess} \label{conditionalT}
Let us consider here a different approach to the general Problem \ref{DATransferVLC}, suppose that we have solved the source problem i.e. \begin{equation}
    f^* \in \argmin_{f \in \mathcal{C}_1} \mathbf{E}_{\pi^s}[L_1(f(X),Y)].
\end{equation}
Similar to the ideas in \cite{Courty} one can make assumptions like Assumption \ref{TransportIndep} in order to benefit from the source sample by using conditional distributions. Given $y \in \spt(\Proj_{2}\# \pi^s)$ and $f\in \mathcal{C}_1$, assume we can find $T^{f,y}$  optimal transport map for the cost function $c_y(x,\Tilde{x}) = \lvert L_1(f^*(x),y) - L_2(f^*(\Tilde{x}),y) \rvert$ between the conditional distributions $\pi^s(x \lvert Y = y)$ and $\pi^t(x \lvert Y = y)$. The question is now how to generate an element in $\mathcal{C}_2$ from the learnt information on the conditional distributions. The first immediate guess is to average with respect to the target distribution, that is if 
\begin{equation*}
    d\pi^t(x,y) = d\pi^t(x \lvert Y=y) d\nu^t(y)
\end{equation*}
a guess for a learnt agent would be 
\begin{equation} \label{GuessAverage}
   f_{ad} = f^* \circ (T^{f^*})^{-1}, \text{ where } T^{f^*}(x) = \int_{Y} T^{f,y}(x) d\nu^t(y).
\end{equation}
In the general case, no estimates on the control of learning for agent \eqref{GuessAverage} are known. \\
It is expected that if the measures satisfy that $d_h$ from \eqref{measureTransfer} is small then the agent obtained using \eqref{GuessAverage} is good although so far no precise statements have been shown. Formula \eqref{GuessAverage} is a reasonable guess because it takes into account the best agent at each $y$ before averaging over all $y \in Y$. 
\begin{question}
In the context of Problem \ref{DATransferVLC}, is it true that if $d_h(\pi^s,\pi^t)$ is small, then $f_{ad}^*$ performs well in \eqref{pitDAProblem} when constructed using \eqref{GuessAverage} and pre-conditioning? Is this performance quantifiable? Is it true that as $n \to \infty $, 
\begin{equation*}
    \mathbf{E}_{\pi_n}[L_2(f_n^*(X),Y)] - \min_{f \in \mathcal{C}_2} [L_2(f(X),Y)] \to 0  ? .
\end{equation*}
Can such performance be studied by $d_h$ of \eqref{measureTransfer} when  $ h(r) = \lvert r \rvert$ ?. \\
How does \eqref{GuessAverage} compare to 
\begin{equation}
    f^* \circ T_2, \text{ where } T_2 = \int_Y (T^{f,y})^{-1} (x) d\nu^t(y)?
\end{equation}
\end{question}
This questions are relevant both in the field of transfer learning and to measure pre-conditioning. \\
The computation of $T^{f,y}$ may be difficult in practice and we expect measure-preconditioning for every $y$ to benefit the performance of the intermediate algorithms without disruption on convergence. Numerical simulations are being performed to corrobate this idea and study the performance of \eqref{GuessAverage} and will appear in subsequent works.
\subsubsection{Data-driven conditional OT}
On \cite{DataDriven} the authors studied the following problem given a cost function $c$ in the product space $X \times Z$ and a probability measure on $X \times Z$: 
\begin{equation} \label{datadrivenot}
    \min_{\substack{T(\cdot,z) \\ \forall z T(\cdot, z) \# \rho( \cdot \lvert z) = \mu(\cdot \lvert z) }} \int c(x,T(x,z)) d\rho(x,z)
\end{equation}
which they denoted the data-driven optimal transport problem. In the same work, the authors showed that the minimization of \eqref{datadrivenot} is equivalent to \begin{equation} \label{dualdatadriven}
    \min_{T( \cdot, z)} \max_{ \lambda \geq 0} \int c(x,T(x,z)) d\rho(x,z) + \lambda \ent\left(  \mu(\cdot \lvert z) \bigg \lvert \frac{1}{2}( T\# \rho(,z) + \mu(\cdot, z))\right)
\end{equation}
The dual formulation of \eqref{datadrivenot} via \eqref{dualdatadriven} already hints a connection with our work. As the algorithm implemented in \cite{DataDriven} is a sequential algorithm using gradient descent, it can be interpreted in the sense of measure pre-conditioners that entropically regularize at every discrete step $n$, just as Definition \ref{EntroBary} in the framework of Wasserstein distance and problem \ref{pinProblem}. This means that an algorithm to compute data-driven conditional optimal transport can benefit directly from measure-preconditioning.
\subsection{Control on optimal transport domain adapted learning}
In this section we present different hypothesis and assumptions that yield stability results on transferred learning. The results are not as strong as those conjectured in section \ref{conditionalT} but directly related to measure pre-conditioning. \\It is evident that there are many options on $\mathcal{C}_1, \mathcal{C}_2,L_2,L_2$ that will ensure the transfer learning is efficient. In this section we reduce to present the most straight-forward formulations. 
\begin{proposition}
 Let $T$ be any map with $ T \# \mu = \nu$ and $ d\pi_1(x,y) = d\pi_2 ( T(x),y) $ if $\mathcal{C}_1 \circ T = \mathcal{C}_2$  then \[
\argmin_{f \in \mathcal{C}_1} \mathbf{E}_{\pi_s} [L(f(X^s),Y^s)] = \argmin_{f \in \mathcal{C}_2} \mathbf{E}_{\pi_t} [L(f(X^t),Y^t)]
\]
\end{proposition}
The proof is a direct consequence of the composition of classes $\mathcal{C}_1 \circ T = \mathcal{C}_2$.
\begin{proposition}
 If $ \mathcal{C}_1 = \mathcal{C}_2 = \mathcal{C}$ and $L_1 = L_2$ and if $\mathcal{C}$ is so that $(x,y) \to L_1(f(x),y)$ is Lipschitz and bounded then for every $f$ \[
\lvert \mathbf{E}_{\pi_1} [L_1(f(x),y)] - \mathbf{E}_{\pi_2} [L_1(f(x),y)]   \rvert \leq d_1(\pi_1,\pi_2)
\]
and so the total loss of transfer learning when the learned agent is adapted is controlled  by the $d_1$-distance between joint measures.
\end{proposition}
\begin{proof}
    The proposition follows directly from the Kantorovich-Rubinstein representation of the $d_1$ norm as $d_1$ is the suprema over Lipschitz functions.
\end{proof}
In \cite{CourtyFlamary} the authors proved the following theorem: 
\begin{theorem} (Courty-Flamary) \\
If $L(x,y) = \lvert x-y \rvert^2$ and $ \mu^s = \frac{1}{n} \sum_{k=1}^n \delta_{x_i^s}$ where $x_1, x_2, \dots, x_n \in \mathbb{R}^n$ and there exist $A$ positive definite matrix and a vector $b$ such that $\mu^t = \frac{1}{n} \sum_{k=1}^{n} \delta_{Ax^s_i + b}$, set $T(x) = Ax + b$ then $f_* \circ T^{-1}$ is a perfect learning agent in the sense that it minimizes \eqref{pitDAProblem}.
\end{theorem}
See in \cite[Theorem 3.1]{CourtyFlamary}. We now generalize this idea before we continue. 
\begin{theorem} Let $\pi_s$, $\pi_t$ be the joint measures for the the source $(X^s,Y^s)$ and $(X^t,Y^t)$ target domains respectively. Denote $\mu_s$ and $\mu_t$ the projections into the $X$-coordinates of $\pi_s$, $\pi_t$ and by $\mu^s_x$ and $\mu^t_x$ the conditional distributions of $Y^s \lvert X^s $ and $Y^t \lvert X^t$.
Assume there exists a map $T:\mathbb{R}^{p} \to \mathbb{R}^p$ such that
\begin{enumerate}
    \item $ T\# \mu_s = \mu_t$
    \item $\mu^t_{T(x)} = \mu_x^s$
    \item $ \mathcal{C}_2 = T \circ \mathcal{C}_1$ 
\end{enumerate}
if $f$ is the solution for \eqref{pinDAProblem} then $f \circ T^{-1}$ is a perfect learner in the sense that it minimizers \eqref{pitDAProblem}.
\end{theorem}
\begin{proof} The proof relies only on the disintegration of measures, as 
\[
\mathbf{E}_{\pi^t}[L(f(X^t),Y^t)] = \int \left( \int L(f(x),y) d\mu^t_x(y)  \right) d\mu_t(x) = \int \left( \int L(f(T(x),y) d\mu^s_{x} \right) d\mu_s(x)
\] 
where we have used the condition $d\mu^t_{T(x)} = d\mu^s(x)$ in the last equality. Minimization over $\mathcal{C}_2$ and the condition $ \mathcal{C}_2 = T \circ \mathcal{C}_1$ yields the result.
\end{proof}
\begin{question}(Can learning error be totally controlled?) \\
Assume $f_*$ minimizes the target problem, under what conditions on $\mu,\nu,L,\mathcal{C}_1,\mathcal{C}_2$ does there exist $ C > 0$ such that
\begin{equation*}
    \bigg \lvert \frac{1}{n} \sum_{i=1}^{n_2}\mathbf{E}[L(f_*(X_i^t),Y_i^t) - L(f_{ad}(X_i^t),Y_i^t) ] \bigg \rvert  \leq C d_2(\mu_s,\mu_t) ? 
\end{equation*}
\end{question}
The previous theorems and the ideas of \cite{Courty} respond this question in very restricted situations. Having a general context to answer this question similar to the one of \ref{flrs} would be essential for the theory of domain adaptation. 
\section{Numerical examples}
In this section we present 2 simple numeric examples using the \texttt{mnist} data set: \begin{enumerate}
    \item First we exemplify the convergence of Theorem \ref{gammaconv} and Proposition \ref{cases} by considering convolutional neural networks applied to a gaussian-filter blurred version of the data set \texttt{mnist}.
    \item In the second part of this section we apply the conditional average guess of section \ref{conditionalT} to try to predict whether an image corresponds to a $6$ or a $7$ using the model trained only on differentiating $1$s from $9$s. The underlying hypothesis is ``that sixes are a lot like nines and ones are a lot like sevens''. We explain what this means and how to use the conditional average guess.
\end{enumerate}
Both experiments can be found publicly in the github repository \verb|joaxchon\slash measure_precon| with the goal of reproducibility. \\
These examples should be understood as ``Toy Examples'' as numerical tests with much more detail and precision will be saved for a work in preparation with several co-authors. The idea of this section is to illustrate the main features of measure pre-conditioning and explain the results of the work.
\subsection{Convergence of agents under gaussian filter blurring}
We consider the \texttt{mnist} data set and use the \texttt{keras} and \texttt{tensorflow} packages to train a convolutional neural net under the modified images. We modify each image by first applying a gaussian filter with variance $\sigma$ to blurr it. At each fixed level of $\sigma$ we obtain a learnt agent $f_{n,\sigma}$ as in Problem \ref{pinProblem} and show that both the losses and the accuracy converge to the agent  $f_{n,0}$ as in Proposition \ref{cases} considering the blurred image as corresponding to the measure $\mu_X \ast \mathcal{N}_{\sigma}$, where $\mathcal{N}_{\sigma}$ denotes the unbiased normal distribution with variance $\sigma$. By proposition \ref{cases} and Theorem \ref{gammaconv} we expect the learnt agents $f_{n,\sigma}$ to get close to $f_{n,0}$, although the precise formulation of the Theorem ensures the convergence as $n \to \infty$ of $\lim_{\sigma \to 0}f_{n,\sigma}$ to $f_{0}$. Note that our theorem also ensures the converges of the weights (in appropriate sense).

\begin{figure}[!htb]
\minipage{0.32\textwidth}
  \includegraphics[width=\linewidth]{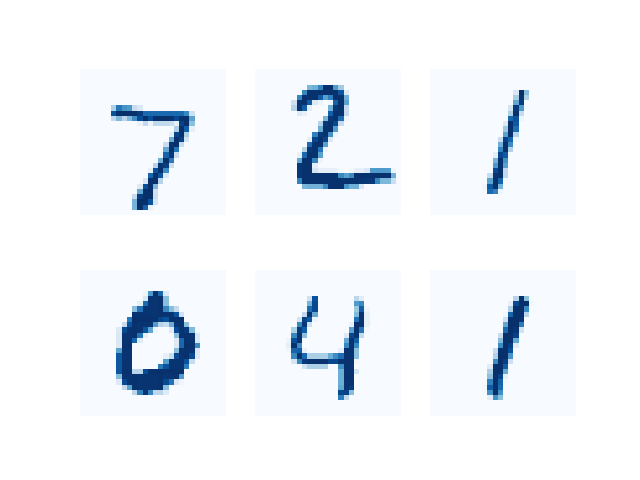}
  \caption{Unblurred, $\sigma = 0$}\label{fig:sigmazero}
\endminipage\hfill
\minipage{0.32\textwidth}
  \includegraphics[width=\linewidth]{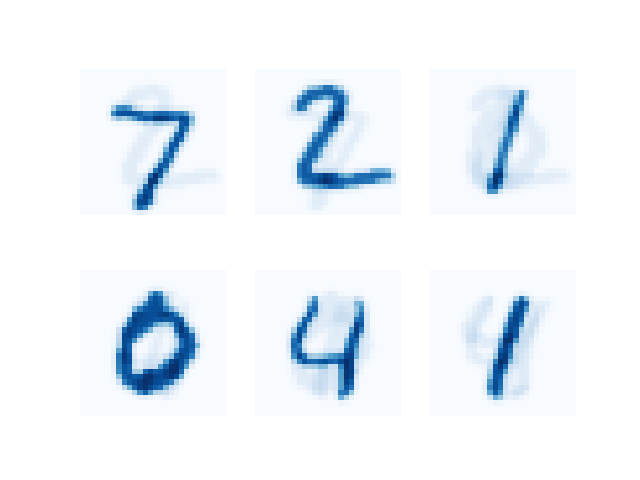}
  \caption{Blurred, $\sigma = 0.5$}\label{fig:sigmamiddle}
\endminipage\hfill
\minipage{0.32\textwidth}%
  \includegraphics[width=\linewidth]{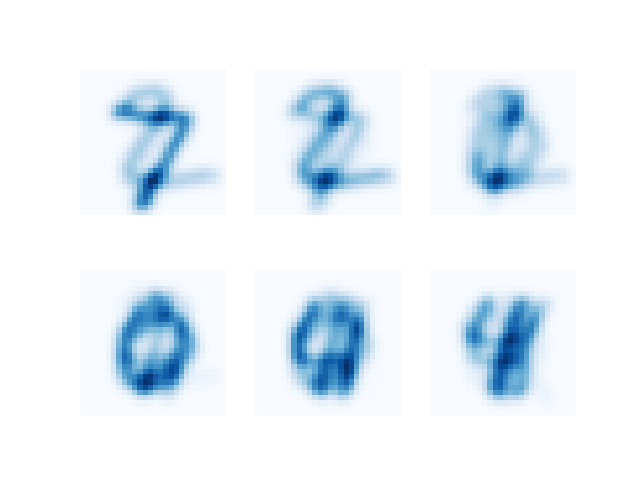}
  \caption{Blurred, $\sigma = 1$}\label{fig:sigma1}
\endminipage
\end{figure}

Proposition \ref{cases} ensures that if we apply a technique for unblurring (Weiner filters, Tychonov's regularization, etc) convergence of the learning agents is guaranteed. We see this result in the following figures:
\newline
\begin{figure}[!htb]
\minipage{0.5\textwidth}
  \includegraphics[width=\linewidth]{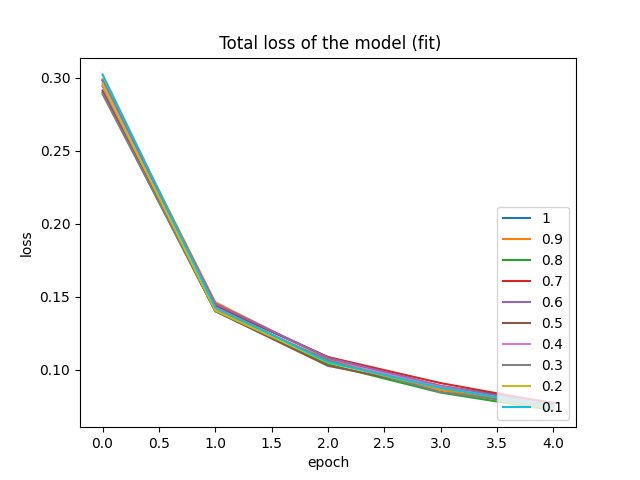}
  \caption{Loss function during training}\label{fig:totalloss}
\endminipage\hfill
\minipage{0.5\textwidth}
  \includegraphics[width=\linewidth]{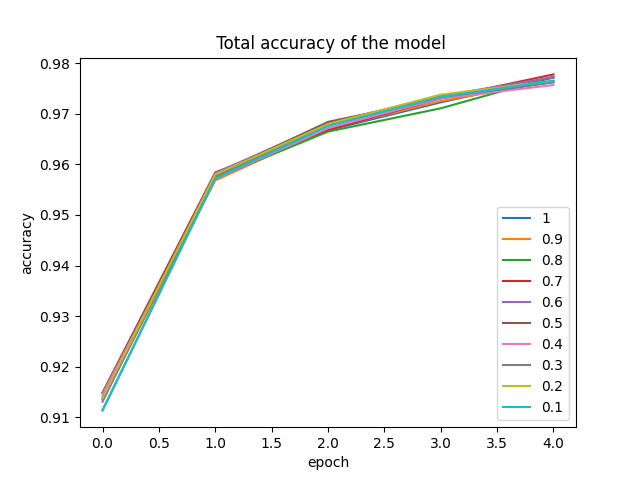}
  \caption{Accuracy of the model during training}\label{fig:totalacc}
\endminipage
\end{figure}
Figures \ref{fig:totalloss} and \ref{fig:totalacc} represent the total accuracy and total loss of the model during training. Different colors correspond to changes on the variance parameter of the gaussian filter. Given the convergence (in weak sense) of the convolutions with gaussians, proposition \ref{cases} indicates the convergence seen in the plots. 
\begin{figure}[!htb]
\minipage{0.5\textwidth}
  \includegraphics[width=\linewidth]{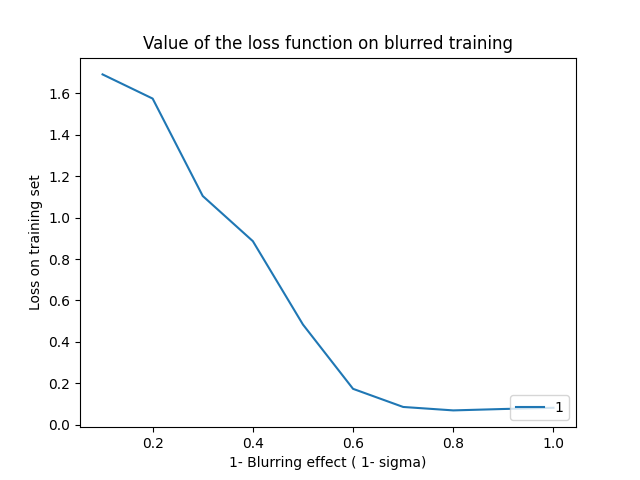}
  \caption{Change in loss during de-blurring}\label{fig:loss}
\endminipage\hfill
\minipage{0.5\textwidth}
  \includegraphics[width=\linewidth]{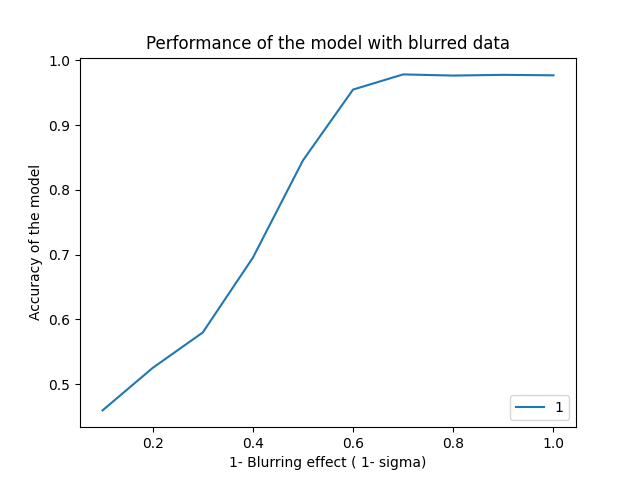}
  \caption{Change in accuracy durin de-blurring}\label{fig:acc}
\endminipage
\end{figure}
\newline
Figures \ref{fig:loss} and \ref{fig:acc} show the behaviour of the agent as we pre-condition (by de-blurring). The change in the training set improves the performance. Note that the conditions of Theorem \ref{gammaconv} ensure that we will see similar behaviour as long as we ensure the method satisfies Definition \ref{flrs}. 
\subsection{The conditional average guess and optimal domain adaptation for 6s and 7s}
In this section we use the \texttt{mnist} dataset to exemplify the technique of measure pre-conditioning for the conditional average guess of section \ref{conditionalT}. The idea is the following: \begin{enumerate}
    \item We train a convolutional neural network only on a data base formed by $1$s and $9$s. The objective is to classify whether a new imput is a $1$ or a $9$. We use sparse cross-entropy as loss function for this step of the learning.
    \item We make the following assumption: The distribution of $1$s is similar to the distribution of $7$s and the distribution of $6$s is similar to that of $9$s.
    \item We use the Python Library \texttt{POT} to approximate optimal transport between the distribution of $1$s and $7s$ and that of $6$s and $9$s. We use $c(x,y) = \lvert x-y \rvert $ as cost function. Independently of the loss function for the model, we note that the cost function $c$ penalizes absolute distance without taking into account the shape, as pre-conditioning we flip every $6$ to make it closer to a $9$.
    \item Finally we use formula \eqref{GuessAverage} as new model to obtain a new agent. 
    \item We test this agent.
\end{enumerate}
Observe that \eqref{GuessAverage} requires the knowledge of $T(x)$ for every $x$ in the support of the measure, nevertheless the computational package can only provide a matching between samples. In order to approximate \eqref{GuessAverage} we approximate via 
\begin{align*}
    f^* = &T^{-1}(\Proj_{\{X_i^1\}}(x)) \cdot \frac{\lvert \lvert x - \Proj_{\{X_i^1\}}(x)  \rvert \rvert }{ \lvert \lvert \Proj_{\{X_i^1\}}(x) + \Proj_{\{X_i^9\}}(x) \rvert \rvert} \\
    & + T^{-1}(\Proj_{\{X_i^9\}}(x)) \cdot \frac{\lvert \lvert x - \Proj_{\{X_i^9\}}(x)  \rvert \rvert }{ \lvert \lvert \Proj_{\{X_i^1\}}(x) + \Proj_{\{X_i^9\}}(x) \rvert \rvert}
\end{align*}
where $\{ X_i^k\}_{i=1}^n$ corresponds to the sample associated to $y = k$ in the training set (the conditional sample). 
\subsubsection{Results}
After the training and the computation of the learnt agent via \eqref{conditionalT}, using the testing data and we obtain that $51.94 \%$ of $7$s were correctly classified by the model and $99.584\%$ of $6s$ were correctly labelled. This is due to the fact that the distribution of $6$'s and the distribution of \textit{flipped} $9$'s is indeed very similar but the distribution of $1$s and $7$s have more differences. This is exactly what we expected as the Wasserstein distance between the distribution of $6$'s and flipped $9$'s is indeed very small, making the conditional guess of \eqref{conditionalT} efficient on identifying $6$0s with the knowledge of $9$'s and $1$'s.
\begin{figure}[!htb]
\minipage{0.5\textwidth}
  \includegraphics[width=\linewidth]{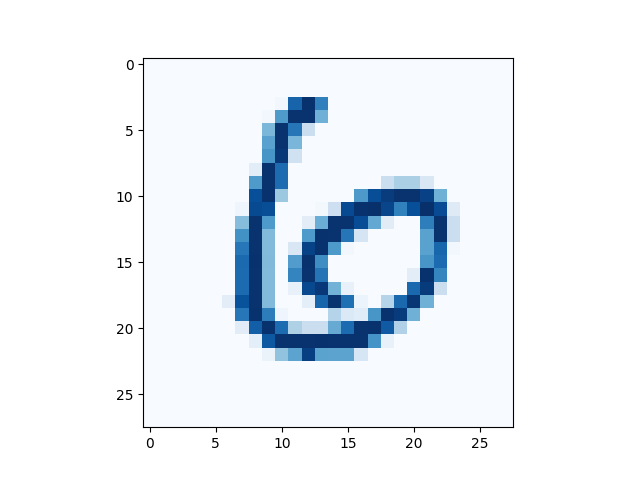}
  \caption{Example of a 6.}\label{fig:6image}
\endminipage\hfill
\minipage{0.5\textwidth}
  \includegraphics[width=\linewidth]{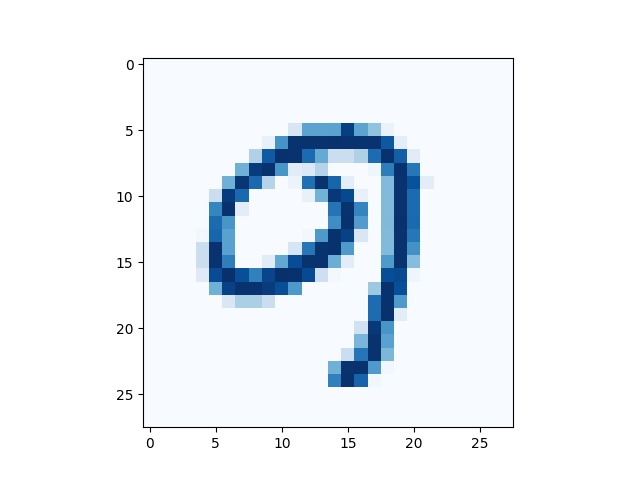}
  \caption{Flipped 6 looks like a 9.}\label{fig:9image}
\endminipage
\end{figure}
\begin{figure}[!htb]
\minipage{0.5\textwidth}
  \includegraphics[width=\linewidth]{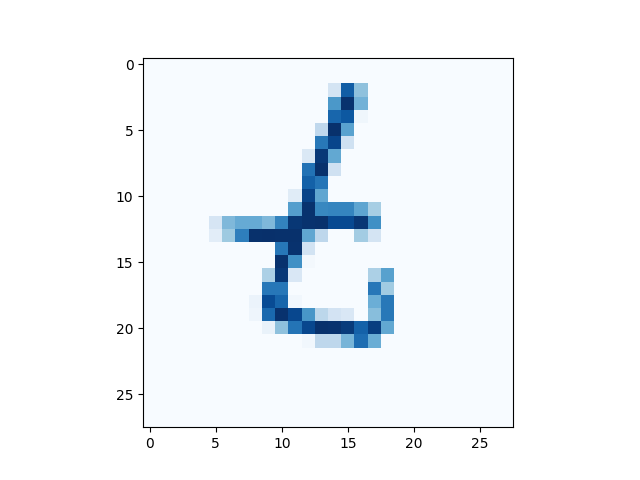}
  \caption{Missclassified $7$.}\label{fig:miss7}
\endminipage\hfill
\minipage{0.5\textwidth}
  \includegraphics[width=\linewidth]{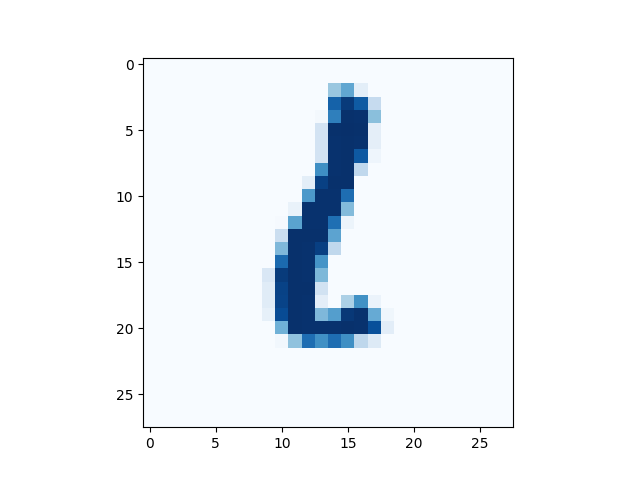}
  \caption{Correctly classified $7$.}\label{fig:correct7}
\endminipage
\end{figure}
\newline
We can see that the agent \eqref{conditionalT} mistakes $7$s for $6$'s when the middle line on the $7$ is big and hence increasing the Wasserstein distance to the distribution of $1$'s.
In a following work in preparation, additional to the flipping method to relate $6$s and $9$s we will pre-condition the distribution of $1$'s by adding noise in a way to make it look more like $7$'s. Theorem \ref{gammaconv} shows that the learnt agent converges to the original one as we reduce the noise to $0$. The script can be found in \verb|joaxchon\slash measure_precon\conditional_average_adaptation.py|
\section{Outside of the framework}
In this section we explain how the framework developed in this article can be extended to encompass more general situations (whose formulation is not exactly represented by \eqref{piProblem} and \eqref{pinProblem}) but benefit from the same ideas. 
\subsection{Using pre-conditioners on WGANs}
The Wasserstein Generative Adversarial Networks (WGAN) introduced in \cite{WGAN} is a generalization of the generative adversarial networks (GAN) introduced in the seminal work \cite{GAN}. The reason to consider the Wasserstein framework is due to the convergence properties of the Wasserstein metric together with the representation of Kantorovich-Rubinstein. The WGAN problem consists in computing 
\begin{equation} \label{WGAN}
    \argmin_{\theta} \argmax_{w \in \mathcal{W}} \mathbf{E}[f_w(X)] - \mathbf{E}[f_{w}(g_{\theta}(Z))]
\end{equation} 
where $X \sim \mathbb{P}_1$ is prescribed, $ Z \sim \mathbb{P}_2$ and $\{g_{\theta}\}_{\theta \in \Theta}$ is a parametric function space.  
Further work would study the same principles applied in this document to the more general version of the problem admitting \eqref{WGAN} using maybe 2 parametric families $\mathcal{C}, \tilde{\mathcal{C}}$. The only difference between our problem and \eqref{WGAN} is the presence of an extra outer minimization problem. It is clear that algorithms like TTC presented in \cite{Milne} that take a dual approach can benefit from sequential measure-pre-conditionining. 
In the original formulation, as in \cite{WGAN}  \begin{equation*}
   \sup_{f \in \mathcal{C}} \int f d \mu - \int f d\nu + - \lambda \int (\lvert \nabla f\rvert - 1 )^2 d\sigma 
\end{equation*} 
where $ Z \sim \sigma $ iff $ Z = t X + (1-t)Y$ where $ t \sim U[0,1]$, note that we can replace $\mu$ and $\nu$ at level $n$ via the empirical measures or measure pre-conditioners. This means that measure-preconditioning can be applied in more general circumstances than Problem \ref{piProblem} as the estimation of $\sigma$ can be done via $ t X_n + (1-t)Y_n$ where $ t \sim U[0,1]$ and the triangle inequality yields convergence.
\subsection{Covariate shift domain adaptation problem }
In general, the label-shift domain adaptation problem is usually written as 
\begin{equation}\label{GeneralDA}
    \min_{h,g} \frac{1}{n}\sum_{i=1}^n L(h(g(x_i^s)),y_i^s) + \lambda \ent(\mu_s^g \rvert \mu_t^g) + \Omega(h,g)
\end{equation}
where $h$ is the hypothesis, $g$ is a representation mapping and $\Omega$ is a regularization term. The first term corresponds to losses in approximation while the second and the third correspond to regularizations. Compared to the framework used in Problems \ref{pinDAProblem} and \ref{pitDAProblem}, \eqref{GeneralDA} is a more general version. Nevertheless, the idea of measure pre-conditiniong can substitute the entropy term by using a sequence of entropic regularizations and $\Omega + L$ can be used as a modified loss function. The difference in algorithmic performance of both approaches is an interesting project.

\subsection{COOT and measure pre-conditioning }
In \cite{COOT}, the following problem was introduced to handle at the same time the disparity between correlated distributions and the data marginals. In the case where $X_i \in \mathbb{R}^p$, authors in \cite{COOT} consider the matrix $X = (X_1, \dots, X_n)$ not only as a sample where the randomness comes form a single distribution but as a doubly-random matrix in the sense that each row is considered a sample and the columns are consider features, in this context let $\mu_{S}$ denote the probability measure associated to samples and $\mu_F$ the associated feature distribution one should perform optimal transport simultaneously in sampling and feature spaces. We expect the techniques of the two previous sections to also work in this context mutandis mutatis.
\section{Researcher's criteria on measure pre-conditioning} \label{Problems}
In section \ref{recipe} we explained what a ML-developer should consider as recipe for applying measure pre-conditing. It explained that each modification of the $n$-level measure had different implications which should be pointed towards some (algorithmic) benefit. In general, it may be difficult to know a-priori exactly what to use and so this (and subsequent) work should be considered as a guideline. 
\subsection{Trade-offs} \label{tradeoffs}
In low-dimensional regimes, absolutely continuous (w.r.t. Lebesgue) tend to behave better, while in higher dimensions highly concentrated measures tend to have better properties, see \cite[Chapter 4]{Invitation}. This is already a hint on what to do, if the problem involved has few features, absolutely continuous measures may improve the performance of the algorithm.  
\section{Conclusions and further work}

Recent work \cite{COOT} has introduced new techniques for domain adaptation, the idea is to optimally match features and samples, it is still open lines of investigation how different measure pre-conditioning techniques would impact the co-optimal transport problem. The features and samples are in general of very different nature for which combining more than one of the techniques of section \ref{Techniques} could improve the performance of the algorithms. For example, it may be the case that features share a structure that can be exploited by a specific technique while the relation between samples may algorithmically benefit from another.
\subsection{Order of convergence}
Establishing that the ML problem gives a full learner recovery system is good in order to know convergence is ensured, in algorithmic practice we need more. We need to study the order of convergence and the imrpovements on this order by Measure pre-conditioners, this work is left for future work and other researchers. 
\subsubsection{Data-driven model changes and convergence}\label{data-drivenchange}
In the start of section \ref{formulation} we asked question \eqref{question3}:
Given a choosing of $\pi_n$'s, could we find sequences $L_n$'s and $\mathcal{C}_n$ so that the computations on the $\mathcal{C}_n-$ problem with loss function $L_n$ associated to $\pi_n$ converge to \ref{piProblem}? Could these problems improve the algorithmic performance? \\
In section \ref{mathframework} we studied conditions on $\mathcal{C}$ and $L$ to ensure Definition \ref{flrs} and consequently Theorem \ref{flrsmin}. The question of how and when to change $\mathcal{C}_n$ and $L_n$ at every step is still open and interesting. A good answer would yield heuristics to change the model given the data in terms of the parametric space, this means to not only change the way we measure the information from the data but also how we learn from it. This line of investigation is left for future work.
\subsection{k-nearest neighbohrs and relation to meta-transport}
\subsubsection{k-nearest neighbors and point-process notation}
The list of empirical estimating probabilities (section \ref{preconditionersandestimators}) is obviously non-exhausting. Algorithmic treatment of data such as $k$-nearest neighbors represent a potentially significant pre-conditioning method. The theory of this algorithms is usually developed through point-processes. The extension of this work to point-processes together with section \ref{data-drivenchange} is a promising area for mathematical theory of learning.
\subsubsection{Meta-transport}
Another recent development in Optimal Transport based machine learning is the development of meta-optimal transport in \cite{Meta}. The basic idea, similar to the basic idea of this document is to present a way to improve the performance of ML-algorithms through pre-working on them. The seminal work \cite{Meta} develops completely algorithmic-focused techniques, as explained in section \ref{tradeoffs}. This work is focus on the underlying structured of pre-condiitoning the samples, the statistics in Wasserstein space and how they impact the outputs of the algorithms. In some way, \cite{Meta} tackles the pre-conditioning/pre-measure pre-conditioning in a different manner, with a clever approach based on numerical algorithms. We expect that a theory similar to the one developed in section \ref{preconditionersandestimators} can also encapsulate the algorithmic pre-conditioning. This can be modelled via point-processes (as it's done for $k$-nearest neighbors).
\subsection{General disintegration estimates}
One can study different conditions on $L,\mathcal{C}, \mu, \pi, Y \lvert X$ such that a convergence similar to Theorem \ref{desintegrateuniform} occurs. This area is particularly technical as disintegration is not a continuous operation with respect to some metrics on spaces of probability measures. Generally, one does not necessarily need to estimate the disintegration but can explore different methods of convergence. An approach to full learner recovery systems (\ref{flrs}) in the special case of assumptions on $Y \lvert X$ would be interesting and related with sections \ref{costP1}, \ref{conditionalT} and literature as to \cite{DataDriven} and references therein.

\subsection{Problem 2 of section \ref{main}} \label{goaloftransfer}
If $L$ can be chosen thinking ahead of the Target problem, choose the cost function by chosing an $L_1$ depending on $L_2$ or viceversa. The idea of the problem is to ensure learning can be transferred by picking the problems with the goal of transferring. A full theory with the approach of training with the goal of transferring would be interesting on it's own.
\subsubsection{Choosing the first Loss function to improve the second}
With the same approach as in Section \ref{goaloftransfer}, if we know that we aim to solve the target problem for $L_2, \pi^t, \mathcal{C}_2$, what loss function $L_1$ should we chose given $\pi^s$ and $\mathcal{C}_1$? Similarly, allow $\mathcal{C}_1$ to be chosen. We should chose $L_1$ in a way that data under $\pi^s$ behave similar to $L_2$ under $\pi^t$.  How one takes the target problem into consideration is an open question. 
\subsection{Choosing the target loss model according to the source}
Assume we have solved Problem \ref{piProblem} with set of features $L_1,\mathcal{C}_1, \pi^s$ and we know there is a distribution (unknown to us) on which we aim to transfer the knowledge, what loss function $L_2$ would ensure good properties of the learnt agent on target space? One can think of an $L_2$ loss function that penalizes the error of the learnt agent and \textit{simultaneously} penalizes the difference between probabilities. This function would take into account that a mistake in the model is not relevant when one knows the error on difference of distributions is big. The $L_2$ loss function could be used to simultaneously control model error with (probability) transfer error.

\end{document}